%% file: arxiv-main.tex
\documentclass[11pt]{article}

\usepackage{jcd}
\usemedgeometry
\usepackage[numbers]{natbib}
\bluehyperref

\input{command.tex}

\usepackage{xr}

\title{Some Robustness Properties of Label Cleaning}
\author{Chen Cheng \and John Duchi}

\begin{document}

\maketitle

\input{abstract}

\input{intro}

\input{preliminaries}
\input{surrogate-consistency}

\input{surrogate-examples}

\input{model-consistency}

\input{numerical}

\input{discussion}

\appendix

\input{proof-ranking}

\input{consistency-proofs}

\input{proof-comparison-consistency}

\input{surrogate-examples-proofs}

\input{knn-appendix}

\input{proof-model-consistency}

\input{proof-misspecified-finite-dim}



\bibliographystyle{abbrvnat}
\bibliography{bib}

\end{document}

%% file: command.tex
\newcommand{\ind}{{1}}

\newcommand{\Ep}{\mathbb{E}}
\newcommand{\Prb}{\mathbb{P}}
\newcommand{\real}{\mathbb{R}} 
\newcommand{\prn}[1]{\left({#1}\right)} 
\newcommand{\brk}[1]{\left[{#1}\right]} 

\newcommand{\constant}{\mathsf{C}}  
\newcommand{\cstMT}{\constant_{\mathsf{MT}}}

\newcommand{\cat}{\mathsf{Cat}}  
\newcommand{\multinom}{\mathsf{Multinom}}  

\newcommand{\sgn}{\mathsf{sgn}}  
\newcommand{\pred}{\mathsf{d}}  

\newcommand{\dstr}{P}
\newcommand{\lr}{^{\textup{lr}}}  
\newcommand{\lre}{^{\epsilon}}  

\newcommand{\sX}{\mathcal{X}} 
\newcommand{\sY}{\mathcal{Y}} 

\newcommand{\tloss}{\ell}  
\newcommand{\sloss}{\varphi}  
\newcommand{\extloss}{\delta_\tloss}
\newcommand{\exsloss}{\delta_\sloss}
\newcommand{\exslossagg}{\delta_{\sloss,A}}
\newcommand{\exslossaggm}{\delta_{\sloss,A_m}}

\newcommand{\epsargmin}[1][]{
  \ifthenelse{\isempty{#1}}{%
    \mathop{\epsilon\textup{-argmin}}
  }{%
    \mathop{\epsilon_{#1}\textup{-argmin}}
  }
}

\newcommand{\sups}[1]{^{({#1})}} 

\newcommand{\knnriskgap}{\delta_{\sloss,n,K}}  

%% file: abstract.tex
\begin{abstract}
  We demonstrate that learning procedures that rely on aggregated labels, e.g.,
  label information distilled from noisy responses, enjoy robustness properties
  impossible without data cleaning.
  This robustness appears in several ways.
  In the context of risk consistency---when one takes the standard approach
  in machine learning of minimizing a surrogate (typically convex) loss in
  place of a desired task loss (such as the zero-one mis-classification
  error)---procedures using label aggregation obtain stronger consistency
  guarantees than those even possible using raw labels.
  And while classical statistical scenarios of fitting perfectly-specified
  models suggest that incorporating all possible information---modeling
  uncertainty in labels---is statistically efficient, consistency fails for
  ``standard'' approaches as soon as a loss to be minimized is even slightly
  mis-specified.
  Yet procedures leveraging aggregated information still converge to optimal
  classifiers, highlighting how incorporating a fuller view of the data
  analysis pipeline, from collection to model-fitting to prediction time,
  can yield a more robust methodology by refining noisy signals.
\end{abstract}

%% file: intro.tex
\section{Introduction}

Consider the data collection pipeline in a supervised learning problem.
Naively, we say that we collect pairs $(X_i, Y_i)_{i = 1}^n$ of features
$X_i$ and labels $Y_i$, fit a model, and away we go~\citep{HastieTiFr09}.
But this belies the complexity of modern datasets~\citep{DengDoSoLiLiFe09,
  KrizhevskyHi09, RussakovskyDeSuKrSaMaHuKaKhBeBeFe15}, which require
substantial data cleaning, filtering, often crowdsourcing multiple labels
and then denoising them.
The crowdsourcing community has intensively studied such data cleaning,
especially in the context of obtaining ``gold standard''
labels~\citep{DawidSk79, WhitehillWuBeMoRu09, WelinderBrPeBe10, Vaughan18,
  PlataniosAlXiMi20}.
We take a complementary view of this process, investigating the
ways in which data aggregation fundamentally and necessarily
improves the consistency of models we fit.

In a sense, this paper argues that label cleaning, or aggregating labels
together, provides robustness that is impossible to achieve
without aggregating labels.
There are two faces to this robustness.
First, we improve consistency of estimation:
when minimizing a surrogate loss (e.g., the multiclass logistic
loss) instead of a \emph{task} loss (e.g., the zero-one error),
procedures that use aggregated labels can achieve consistent and optimal
prediction in the limit when this is impossible without data aggregation.
Second, even in finite-dimensional statistical problems, this
aggregation can provide consistent classifiers when
standard methods fail.


Important contributions to the theory of surrogate risk consistency trace to
the 2000s~\citep{Zhang04b,LugosiVa04, Steinwart07}, with
\citet*{BartlettJoMc06} characterizing when fitting a model using a convex
surrogate is consistent for binary classification for the zero-one error.
%
%
Since this work, there has been an abundance of work on
surrogate risk consistency, including on multi-label
classification~\citep{Zhang04, TewariBa07, GaoZh11, ZhangAg20,
  AwasthiFrMaMoZh21}, ranking problems~\citep{DuchiMaJo10, DuchiMaJo12,
  PiresSzGh13}, structured prediction~\citep{OsokinBaLa17, CabannesRuBa20,
  NowakBaRu20}, ordinal regression~\citep{PedregosaBaGr17}, restricted hypothesis class~\citep{AwasthiMaMoZh22,
  AwasthiMaMoZh22b, MaoMoZh24, MaoMoZh25} and general theory~\citep{Steinwart07}.
On the one hand,
these analyses, which consider the standard supervised learning scenario of
data pairs $(X, Y)$,  enable us to fully exploit the entire
statistical theory of empirical processes~\citep{VanDerVaartWe96,
  BartlettBoMe05, Koltchinskii06a, BartlettJoMc06}.
On the other, they do not address the data aggregation machinery now
common in modern dataset creation.

It is thus natural to ask about the interaction between consistency
and data aggregation---to begin with, do we need to aggregate at all? If we
can achieve surrogate consistency without data aggregation, we should
perhaps just rely on our mature theoretical
understanding of processes with $(X,Y)$ pairs.
Going one step further, if we aggregate, does aggregation help
consistency, and in what sense does it help?
These main questions motivate this paper.

To further underpin the importance of studying consistency and aggregated
labels, we propose concrete examples---in ranking and binary- and
multiclass-classification with linear estimators---where estimators using
only pairs $(X, Y)$ necessarily fail, but label aggregation methods yield
consistency.
We develop new notions and theory for surrogate consistency with data
aggregation.
In fully nonparametric scenarios, we show how the
number of samples aggregated combine with noise conditions to
improve consistency.
Aggregation will also allow us to demonstrate surrogate risk consistency
under only weak conditions
the surrogate loss;
in the language of the field, losses using aggregated labels
admit (approximate) linear comparison inequalities.
%
Additionally, in contrast to conventional risk consistency theory, which
requires taking a hypothesis class $\mc{F}$ consisting of all measurable
functions, we will show results in classification problems where aggregating
labels guarantees consistency even over restricted hypothesis classes, which
may fail without aggregation.


%% file: preliminaries.tex
\section{Preliminaries}

We first review classical surrogate risk minimization.
Let $\sX$ be the input space and $\sY$ be the output space, with
data $(X_i,Y_i)_{i = 1}^n \in \sX \times \sY$ drawn i.i.d.\ $\dstr$.
Consider learning a scoring function $f: \mathcal{X} \to
\mathbb{R}^d$ that maps an input $x \in \sX$ to a score $s \in \mathbb{R}^d$
for some $d \geq 1$, where a decoder $\pred: \mathbb{R}^d \to \sY$
determines the final prediction via $\what{y} = \pred \circ f(x)$.
Given a loss $\tloss: \sY \times \mathcal{Y} \to \R_+$ and
hypothesis class $\mathcal{F}$, the goal is to minimize
the \emph{task risk} over $f \in \mathcal{F}$
\begin{align}
  R(f) := \mathbb{E}_P \brk{\tloss(\pred \circ f(X), Y)}.
  \label{eq:task-loss-problem-basic}
\end{align}
For example, in binary classification, $d = 1$,
$\pred(s) = \sgn(s)$, and $\tloss(y, y') = \indic{yy' \le 0}$,
yielding $R(f) = \P(Y f(X) \le 0)$.
The challenge of minimizing $R(f)$ is that the task loss $\tloss$ can be
nonsmooth, nonconvex, and---even more---uninformative: the loss landscape of
the $0$-$1$ loss is flat almost everywhere.
This makes even practical (e.g., first-order) optimization impossible.
We will consider a slightly more sophisticated
version of the problem~\eqref{eq:task-loss-problem-basic},
where instead of the loss $\tloss$ being defined only
in terms of the instantaneous label $Y$, we will allow it to depend
on $P(Y \in \cdot \mid X)$, so that
we investigate
\begin{equation}
  \label{eq:task-loss-problem}
  R(f) \defeq \E_{P}\left[\tloss(\pred \circ f(X), P(\cdot \mid X))
    \right],
\end{equation}
whose minimizers frequently coincide with the original
problem~\eqref{eq:task-loss-problem}, but which allows more sophistication.
(For example, in multiclass classification, $Y \in \{1, \ldots, k\}$,
and taking $\tloss(\what{y}, P) = \sum_{y} P(Y = y) \indics{\what{y} \neq
y}$, the risk coincides with the standard 0-1 error rate.)

Instead of the task loss $\tloss$, we thus consider an easier to optimize
surrogate $\sloss: \R^d \times \mc{Y} \to \R$.
Then rather than attacking the risk~\eqref{eq:task-loss-problem} directly,
we minimize surrogate risk
\begin{equation*}
  R_\sloss(f) := \mathbb{E}_P \brk{\sloss(f(X), Y)}.
\end{equation*}
For this to be sensible, we must exhibit some type of consistency with the
task problem~\eqref{eq:task-loss-problem}.
In this paper, we particularly study in two scenarios, which
we will make more formal:
\begin{enumerate}[(i)]
\item The ``classical'' case of Fisher consistency, where $\mc{F}$ contains
  all Borel functions;
\item Statistical scenarios in which the hypothesis class $\mc{F}$ is
  parametric but may be mis-specified.
\end{enumerate}
Our main message is that label aggregation improves
consistency in both scenarios, demonstrating the robustness of label
cleaning.

\subsection{Label aggregation}

Instead of obtaining $(X_i, Y_i)$ pairs, consider
the case that we replace the output $Y$ with a more abstract
variable $Z \in \mathcal{Z}$.
For example, in the motivating scenario in the introduction in which we
collect multiple (say, $m$) noisy labels for each example $X$, we
take $Z = (Y_1, \ldots, Y_m) \in \mc{Y}^m$.
\citet{MaoMoZh24} consider similar multilabeled 
data structure for multilabel losses, while we focus on the label aggregation pipeline.
For an abstract ``aggregation space'' $\mc{A}$, let $A : \mc{Z} \to \mc{A}$
be an aggregating function (e.g., majority vote),
and
let $\sloss : \R^d \times \mc{A} \to \R_+$ be a surrogate loss
defined on this aggregation space.
We then define the aggregated surrogate risk
\begin{align}
  R_{\sloss,A}(f) := \mathbb{E} \brk{\sloss(f(X),
    A(Z))}, \label{eq:agg-surogate-loss-problem}
\end{align}
asking when minimizing the surrogate
problem~\eqref{eq:agg-surogate-loss-problem} is sufficient to
minimize the actual task risk~\eqref{eq:task-loss-problem}.
%
%
Two concrete examples may make this clearer.

\begin{example}[Majority vote] \label{example:mv}
  In the repeated sampling regime, data collection takes the
  form $Z=(Y_1,\cdots,Y_m)$, $Y_i \mid X = x \simiid
  \dstr_{Y \mid X=x}$.
  Define
  $A_m(Z)$ to be the 
  empirical minimizer
  \begin{align*}
    A(Z) = A(\{Y_1, \ldots, Y_m\})
    =\argmin_{y \in \mc{Y}}
    \sum_{l = 1}^m \tloss(y, Y_l). 
  \end{align*}
  When $\tloss(\what{y}, y) = \indics{y \neq \what{y}}$,
  this corresponds exactly to majority vote; the more
  general form allows more abstract procedures.
\end{example}

We can also (roughly) capture $K$-nearest neighbor aggregation procedures:

\begin{example}[$K$-nearest neighbors] \label{example:knn}
  Consider an abstract repeated sampling scenario in which an example $X$
  comes with a label $Y$ and an additional draw $(X_i, Y_i)_{i = 1}^m
  \simiid P$, where $m$ is the number of additional examples, so $Z = (Y,
  (X_i, Y_i)_{i = 1}^m)$.  Let $\dist: \sX \times \sX \to \R_+$ be a
  distance metric on $\sX$.
  Let $\{X_{(1)}, \ldots, X_{(m)}\}$ order the input sample $\{X_i\}_{i=1}^m$
  by distance, $\dist(X,X_{(1)}) \leq \ldots \leq \dist(X,X_{(m)})$ (and let
  $X_{(0)} = X$).
  For $K \geq 0$, we can aggregate the
  $K$-nearest neighbors of $X$, for example, by choosing
  \begin{align*}
    A_{m,K}(Z) \defeq
    \argmin_{y \in \mc{Y}} \sum_{l = 0}^K \tloss(y, Y_{(l)}). 
  \end{align*}
  In Appendix~\ref{sec:discussion-knn}, we leverage the results in the
  coming sections to move beyond this population-level scenario to address
  aggregation from a single sample $(X_i, Y_i)_{i = 1}^n$.
\end{example}

%% file: surrogate-consistency.tex
\section{Surrogate consistency}

The standard framework for surrogate consistency~\citep{Steinwart07}
assumes that $\mc{F}$ consists of all Borel
measurable functions $f: \mathcal{X} \to \mathbb{R}^d$.
%
%
Working in the abstract setting in the preliminaries, define the
conditional task risk $R(s \mid x)$ and the conditional surrogate risk
$R_\sloss(s \mid x)$, $s \in \mathbb{R}^d$ by
\begin{align*}
  R(s \mid x)
  \defeq \tloss(\pred \circ s, P(Y \in \cdot \mid X = x))
  ~~ \mbox{and} ~~
  R_\sloss(s \mid x) := \mathbb{E}\brk{\sloss(s, Y) \mid X = x}.
\end{align*}
We then define the pointwise excess risks
\begin{align*}
  \extloss(s, x) \defeq R(s \mid x) - \inf_{s' \in \R^d} R(s' \mid x) ,
  \qquad
  \exsloss(s, x) \defeq R_\sloss(s \mid x) -
  \inf_{s' \in \R^d} R_\sloss(s' \mid x),
\end{align*}
as well as the minimal risks $R\opt := \inf_{f \in \mathcal{F}} R(f)$ and
$R_\sloss\opt := \inf_{f \in \mathcal{F}} R_\sloss(f)$.
We follow the standard~\citep{Steinwart07,BartlettJoMc06,Zhang04b}
that consistency requires at least (i) Fisher consistency and,
if possible, a stronger and quantitative (ii) uniform comparison inequality:
respectively, that for all data distributions $\dstr$,
\begin{enumerate}[(i)]
\item \label{item:pointwise-comparison}
  For any sequence of functions $f_n \in \mathcal{F}$, $R_\sloss(f_n)
  \to R_\sloss\opt$ implies $R(f_n) \to R\opt$.
\item \label{item:uniform-comparison}
  For a non-decreasing $\psi: \R_+ \to \R_+$, $\psi (R(f) - R\opt) \leq
  R_\sloss(f) - R_\sloss\opt$ for all $f \in \mathcal{F}$, where $\psi$
  satisfies $\psi(\epsilon) > 0$ for all $\epsilon > 0$,
\end{enumerate}

In the case of binary classification when $\sloss$ is margin-based and
convex, the two consistency notions coincide~\citep{BartlettJoMc06}.
The stronger uniform guarantee~\eqref{item:uniform-comparison} need not
always hold, the \emph{calibration function} $\wb{\psi}$ provides a
canonical construction through the excess risk:
\begin{align*}
  \psi(\epsilon, x)
  \defeq \inf_{s \in \R^d}
  \left\{\exsloss(s, x) \mid \extloss(s, x) \ge \epsilon\right\}
  ~~ \mbox{and} ~~
  \wb{\psi}(\epsilon)
  \defeq \inf_{x \in \mc{X}} \psi(\epsilon, x).
\end{align*}
Consistency and comparison inequalities follow
from the calibration functions (see~\citet[Prop.~25]{Zhang04a} and
\citet[Thm.~2.8 and Lemma~2.9]{Steinwart07}):
\begin{corollary}
  \label{cor:comparison-inequality-classical}
  The surrogate $\sloss$ is Fisher consistent~\eqref{item:pointwise-comparison}
  for $\tloss$ if and only if $\psi(\epsilon, x) > 0$ for all
  $x \in \mc{X}$ and $\epsilon > 0$.
  Let $\psi$ be the Fenchel biconjugate of $\wb{\psi}$.
  Then $\wb{\psi}(\epsilon) > 0$ if and only
  if $\psi(\epsilon) > 0$, and for all measurable $f$,
  \begin{align*}
    \psi(R(f) - R\opt) \leq R_\sloss(f)  - R_\sloss\opt.
  \end{align*}
\end{corollary}

In the general risk minimization problem~\eqref{eq:task-loss-problem} we
would like at least a Fisher consistent~\eqref{item:pointwise-comparison}
surrogate for $\tloss$, so that minimizing $R_\sloss(f) = \E[\sloss(f(X),
  Y)]$ would imply minimizing $R(f)$.
Given such a result, using only paired observations $(X, Y)$ rather than
tuples $(X, Y_1, \ldots, Y_m)$, we could bring the entire theory of
empirical processes and related statistical tools~\citep{VanDerVaartWe96,
  BartlettBoMe05, Koltchinskii06a, BartlettJoMc06} to bear on the problem.
Moreover, data collection procedures would be simpler, necessitating only
single pairs $(X, Y)$ for consistent estimation.
Unfortunately, such results are generally impossible, as we detail in the
next extended example, necessitating the necessity of a theory of
aggregation that we pursue in Sec.~\ref{sec:consistency-agg}.

\subsection{Fisher consistency failure without label aggregation: ranking}
\label{sec:counter-example}

Consider the problem of ranking $k$ items using pairwise comparison
data~\citep{Keener93, DworkKuNaSi01, DuchiMaJo12, NegahbanOhSh16},
where the space $\mathcal{Y}$ consists of all pairwise comparisons of these
items, $\mc{Y} = \{(i,j): i \neq j, 1 \leq i, j \leq k\}$, i.e. $(i,j)$ denotes the $i$-th item has higher ranking than the $j$-th item.
The (population) rank aggregation problem is, for each $x$,
to transform the
probabilities $p_{ij} = P(Y = (i,j) \mid x)$ into a
ranking of the $k$ items.
While numerous possibilities exist for such aggregation, we consider a
simple comparison-based aggregation scheme~\citep[cf.][]{Keener93}; similar
negative results to the one we show below hold for more sophisticated
schemes.
Define the normalized transition matrix $C_x \in \R_+^{k \times
  k}$ with entries $(C_x)_{ii} = 0$ and
\begin{align*}
  (C_x)_{ij} = 
  \frac{p_{ij}}{\sum_{l \neq j} p_{lj}}
  ~~ \mbox{for~} i \neq j,
\end{align*}
where we let $0/0 = 1/(k-1)$ so that $C_x$ is stochastic, satisfying $C_x^T
\ones = \ones$.
We then rank the items by the vector $C_x \ones \in \R_+^k$,
which measures how often a given item is preferred to others.
(One may also
take higher powers $C_x^p \ones$ or Perron vectors~\citep{Keener93}; similar
results to ours below hold in such cases.)
Tacitly incorporating the decoding $\pred$ into the task loss $\tloss$,
we
\begin{equation*}
  \tloss(s, C) \defeq
  \max_{i < j} \indic{ (s_i - s_j) (e_i - e_j)^T C \ones \le 0,
  ~ (e_i - e_j)^T C \ones \neq 0 },
\end{equation*}
which penalizes mis-ordered scores between $s$ and $C$.
The population task risk~\eqref{eq:task-loss-problem} is thus
\begin{align}
  R(f) & \defeq
  \P\left(f(X) \text{ and } C_X\ones \text{ order differently}\right)
  \label{eq:task-risk-ranking}
\end{align}


Now consider a convex surrogate $\sloss: \R^k \times \mathcal{Y} \to
\R$.
We restrict to $s \in \R^k$ for which $s^T\ones = 0$, a
minor restriction familiar from multiclass classification
problems~\citep{Zhang04a,TewariBa07}, which is natural as for decoding a
ranking we require only the ordering of the $s_i$.
Unfortunately, there is no convex Fisher consistent surrogate for the
problem~\eqref{eq:task-risk-ranking} (see
Appendix~\ref{proof:failure-ranking}).
\begin{proposition}
  \label{proposition:failure-ranking}
  Consider the ranking problem with task risk~\eqref{eq:task-risk-ranking}
  over $k \ge 3$ outcomes.
  If $\sloss : \R^k \times \mc{Y} \to \R$ is convex in its first argument,
  it is not Fisher consistent.
\end{proposition}

Nonetheless, a reasonably straightforward argument yields consistency when
we allow aggregation methods as soon as $m$, the number of collected
comparisons, satisfies $m \ge k$.
The idea is simple: we regress predicted scores $f(x)$ on frequencies
of label orderings.
We assume multiple independent pairwise comparisons $Z = (Y_1, \ldots, Y_m)$
conditioned on $X$, and letting $m_{ij} = \sum_{y \in \mc{Y}^m} \indic{y =
  (i,j)}$ and $m_j = \sum_{i = 1}^k m_{ij}$, we define
the aggregation
\begin{align*}
  A(Z) & = \begin{cases}
    \star, & \mbox{if} ~ m_j = 0 \text{ for some }j\text{ in }[k], \\
    (\frac{m_{i1}}{m_1} + \frac{m_{i2}}{m_2} + \cdots + \frac{m_{ik}}{m_k})_{
      i \in [k]}
    & \mbox{otherwise, i.e.~if}~ m_j > 0 \text{ for all }j \in [k].
  \end{cases}
\end{align*}
Regressing directly on $A(Z)$ when $A(Z) \neq \star$ yields consistency, as
the next proposition demonstrates (see
Appendix~\ref{sec:proof-success-ranking} for a proof):
\begin{proposition}
  \label{proposition:success-ranking}
  Define $\sloss(s, q) = \ltwo{s - q}^2$ for $s, q \in \R^k$ and
  $\sloss(s, \star) = 0$.
  Then if $m \ge k$, $\sloss$ is Fisher consistent
  for the ranking risk~\eqref{eq:task-risk-ranking}.
\end{proposition}
We point out the reason Proposition~\ref{proposition:failure-ranking} fails to guarantee consistency, whereas Proposition~\ref{proposition:success-ranking} does ensure consistency, lies in the expansion of the response space from $\mc{Y}$ to $\R^k$ induced by aggregating multiple labels.


\subsection{Label aggregation obtains stronger surrogate consistency}
\label{sec:consistency-agg}

The extended ranking example in ranking suggests potential benefits of
aggregating labels, and it is natural to ask how aggregation interacts with
surrogate consistency more generally.
Thus, we present two results here: one that performs an essentially
basic extension of standard surrogate-risk consistency, and the second
that shows how aggregation-based methods can ``upgrade'' what might
nominally be inconsistent losses into consistent losses,
as Proposition~\ref{proposition:success-ranking} suggests may be possible.
%

\subsubsection{Basic extensions of surrogate consistency}

We begin by making the more or less obvious generalization
of calibration functions for standard cases, extending
the classical comparison
inequalities in Corollary~\ref{cor:comparison-inequality-classical}.
For an arbitrary aggregation method $A : \mc{Z} \to \mc{A}$, define the
conditional surrogate risk with data aggregation
\begin{align*}
  R_{\sloss,A}(s\mid x) \defeq \E\brk{\sloss(s, A(Z)) \mid X = x}.
\end{align*}
As in the non-aggregated case, the pointwise excess risk
\begin{align*}
  \exslossagg (s,x) := R_{\sloss,A}(s\mid x)
  - \inf_{s \in \R^d} R_{\sloss,A}(s\mid x)
\end{align*}
then defines the pointwise and uniform calibration functions
\begin{equation}
  \label{eqn:pointwise-agg-calibration}
  \wb{\psi}_A(\epsilon, x) \defeq \inf_{s \in \R^d}
  \left\{\exslossagg(s, x) \mid \extloss(s, x) \geq \epsilon\right\}
  ~~ \mbox{and} ~~
  \wb{\psi}_A(\epsilon) \defeq \inf_{s \in \R^d} \wb{\psi}_A(\epsilon, x).
\end{equation}
A consistency result then follows, similar to
Corollary~\ref{cor:comparison-inequality-classical}, under appropriate
measurability conditions (we will leave these tacit as they are not central
to our results).
Then more or less as a corollary of \citet[Thm.~2.8]{Steinwart07}, we have
the following consistency result.
(We include a proof for completeness in
Appendix~\ref{sec:proof-comparison-inequality-agg}.)

\begin{proposition}
  \label{prop:comparison-inequality-agg}
  Assume there exists $b : \mc{X} \to \R_+$ with
  $\int b(x) dP(x) < \infty$ such that
  $\extloss(f(x), x) \le b(x)$.
  The surrogate $\sloss$ is Fisher consistent~\eqref{item:pointwise-comparison}
  for the task risk~\eqref{eq:task-loss-problem}
  if and only if $\wb{\psi}_A(\epsilon, x) > 0$ for all $x \in \mc{X}$
  and $\epsilon > 0$.
  Additionally, if $\psi_A = (\wb{\psi}_A)^{**}$ is the Fenchel
  biconjugate of $\wb{\psi}_A$,
  then
  \begin{align*}
    \psi_A(R(f) - R\opt) \leq R_{\sloss,A}(f) - R_{\sloss,A}\opt.
  \end{align*}
\end{proposition}

The result captures the classical consistency guarantees---nothing
particularly falls apart because of aggregation---but it provides no
specific guarantees of improved consistency.
We turn to this now.

\subsubsection{Identifying surrogates and consistency}

We now turn under essentially minimal conditions on the surrogate, there is
a generic aggregating strategy that (asymptotically in the number of
observations $y$) guarantees consistency for any task loss that seeks to
minimize $\tloss(f(x), y)$, i.e.,
$R(f) = \E[\tloss(f(X), Y)]$.
We assume that $\card(\mc{Y}) = k < \infty$\footnote{We emphasize that the example discussed in Sec.~\ref{sec:counter-example} belongs to this regime, where $k$ denotes the ranking length and $\card(\mc{Y}) = k(k-1)$. With a slight abuse of notation, we will henceforth use $k$ to denote the cardinality of the label space.},
and we impose a minimal identifiability assumption on the surrogate
loss.

\begin{definition}[Identifying surrogate]
  \label{def:feasible}
  A surrogate $\sloss : \R^d \times \mc{A} \to \R$ is
  \emph{$(\constant_{\sloss, 1}, \constant_{\sloss, 2})$-identifying
  for $\mc{Y}$},
  $0 < \constant_{\sloss,1} \le \constant_{\sloss,2} < \infty$
  if there exist
  $\{a_y\}_{y \in \mc{Y}} \subset \mathcal{A}$ and
  vectors $\{s_y\}_{y \in \mc{Y}} \subset \real^d$ such that
  $\pred(s_y) = y$ and for which for all $y \neq y'$,
  \begin{subequations}
    \begin{align}
      \sloss(s_y, a_y) + \constant_{\sloss, 1}
      & \leq \inf_{\pred \circ s \neq y} \sloss(s, a_y),
      \label{eq:feasible-surrogate-1} \\
      \sloss(s_y, a_{y'}) - \constant_{\sloss, 2}
      & \leq \inf_{s \in \R^d} \sloss(s, a_{y'}). \label{eq:feasible-surrogate-2}
    \end{align}
  \end{subequations}
\end{definition}
\noindent
Inequality~\eqref{eq:feasible-surrogate-1} captures
that for each class $y \in \mc{Y}$,
there exists a parameter $a_y \in \mc{A}$ such that the minimizer
of $\sloss(\cdot, a_y)$ identifies $y$.
A finite
$\constant_{\sloss, 2}$ exists for~\eqref{eq:feasible-surrogate-2}
if $\sloss(\cdot, a)$ has a finite lower bound.
Notably, Definition~\ref{def:feasible}
does not require that $\sloss(\cdot, a)$ is convex or that it is
consistent when $\mc{A} = \mc{Y}$ and $Z = Y$, i.e., without label
aggregation. Another benefit is that we also do not require the decoder $\pred$, such as requiring it to coincide with the pointwise $\argmax$ operator.

\begin{example}
  Consider the binary hinge loss $\sloss(s,a) = \max \{1 - sa, 0\}$ for
  $\mc{A} = \mc{Y} = \{\pm 1\}$.
  For $y \in \{-1, 1\}$, take $a_y = s_y = y$, so that
  $\sloss(s_1, a_1) = \sloss(s_{-1}, a_{-1}) = 0$, while $\inf_{sa \le 0}
  \sloss(s, a) = 1$.
  Similarly, $\sloss(s_1, a_{-1}) = \sloss(s_{-1}, a_1) = 2$, so the hinge
  loss is $(1, 2)$-identifying.
\end{example}

We defer further concrete examples to Sec.~\ref{sec:surrogate-example} where we study the consistency amplification by label aggregation. Given an identifying surrogate with parameters $\{a_y\}_{y \in \mc{Y}}$, we
consider a naive aggregation strategy: the generalized majority vote
\begin{equation}
  \label{eqn:naive-mv-aggregation}
  A_m(y_1, \ldots, y_m) \defeq a_{\hat{y}}
  ~~ \mbox{for} ~~
  \hat{y} = \argmin_{y \in \mc{Y}}
  \sum_{i = 1}^m \tloss(y, y_i)
\end{equation}
(breaking ties arbitrarily).
As $m \to \infty$, because $\mc{Y}$ is finite, whenever
$Y_i$ are i.i.d.\ there necessarily exists a (random) $M < \infty$ such
that $m \ge M$ implies
\begin{equation*}
  \argmin_{y \in \mc{Y}}\bigg\{ \sum_{i = 1}^m \tloss(y, Y_i)\bigg\}
  \subset y\opt(x) \defeq \argmin_{y \in \mc{Y}}
  \E\left[\tloss(y, Y) \mid X = x\right].
\end{equation*}
From this, we expect that as $m \to \infty$, the
surrogate $\sloss(\cdot, A_m)$ ought to be consistent.
In fact, we have the following corollary of our coming results,
guaranteeing (asymptotic) consistency:
\begin{corollary}
  \label{corollary:always-consistent}
  Let $m = m(n) \to \infty$
  and $\sloss$ be identifying (Def.~\ref{def:feasible}).
  Then
  \begin{equation*}
    R_{\sloss,A_m}(f_n) - R\opt_{\sloss, A_m} \to 0
    ~~ \mbox{implies} ~~
    R(f_n) - R\opt \to 0.
  \end{equation*}
\end{corollary}

\subsubsection{Identifying surrogates and consistency amplification}

In cases with low noise in the labels, the aggregation
strategy~\eqref{eqn:naive-mv-aggregation} allows an explicitly improved
comparison inequality $\psi(R(f) - R\opt) \le R_\sloss(f) - R\opt$, in that
$\psi$ is linear over some range of $\epsilon > 0$---and linear growth is
the strongest comparison inequality possible~\citep{OsokinBaLa17,
  NowakBaRu20}.
More generally, strict comparison inequalities, such as those
present in Proposition~\ref{prop:comparison-inequality-agg}, can be too
narrow, as it can still be practically convenient to adopt
inconsistent surrogates~\citep{Liu07, OsokinBaLa17, NowakBaRu20}.
Thus, we follow \citet{OsokinBaLa17} to introduce \emph{$(\xi, \zeta)$
consistency}, which requires a comparison function $\psi$ to grow linearly
only for $\epsilon \ge \xi$, so that the surrogate captures a sort of ``good
enough'' risk.
\begin{definition}
  \label{definition:almost-linear-consistency}
  The surrogate loss $\sloss$ and aggregator $A$ yield \emph{level-$(\xi,
  \zeta)$ consistency} if there exists $\psi$ satisfying $\psi(\epsilon) \ge
  \zeta \epsilon$ for $\epsilon \ge \xi$, and $\psi(R(f) - R\opt) \le
  R_{\sloss, A}(f) - R\opt_{\sloss,A}$.
\end{definition}
\noindent
In the following discussion, we show under minimal assumptions, label
aggregation~\eqref{eqn:naive-mv-aggregation}
can achieve level-$(o_m(1), \zeta)$ consistency even if the
surrogate is Fisher inconsistent.

We introduce a quantifiable noise condition, adapting the
now classical Mammen-Tsybakov noise conditions~\citep{MammenTs99}
(see also~\citet{BartlettJoMc06}).
Define
\begin{equation}
  \label{eqn:minimal-wrong-excess-risk}
  \Delta(x) \defeq \min_{\pred(s) \not\in y\opt(x)} \extloss(s, x),
\end{equation}
the minimal excess conditional risk when making an incorrect prediction. 
In binary classification problems with
$\mc{Y} = \{\pm 1\}$, one obtains
$\Delta(x) = |2 P(Y = 1 \mid X = x) - 1|$, and more generally,
we expect that consistent estimation should be harder when
$\Delta(x)$ is closer to 0.
%
%
We can define the Mammen-Tsybakov conditions
(where the constant $\cstMT > 0$ may change) as
\begin{equation}
  \label{eqn:mammen-alpha}
  \P\left(\pred \circ f \neq \pred \circ f\opt\right)
  \le \cstMT \left(R(f) - R\opt\right)^\alpha
  ~~ \mbox{for~all~measurable}~ f,
  \tag{N{$_\alpha$}}
\end{equation}
where we refer to condition~\eqref{eqn:mammen-alpha}
as having \emph{noise exponent $\alpha$},
and
\begin{equation}
  \label{eqn:mammen-beta}
  \P(\Delta(X) \leq \epsilon) \leq (\cstMT  \epsilon)^\beta
  ~~ \mbox{for} ~\epsilon > 0.  
  \tag{M$_\beta$}
\end{equation}
\noindent
Here, $\alpha \in [0, 1]$ and $\beta \in [0, \infty]$, so that
conditions~\eqref{eqn:mammen-alpha} and \eqref{eqn:mammen-beta} always
trivially hold with $\alpha = \beta = 0$, moreover, as in the binary
case~\citep[Thm.~3]{BartlettJoMc06}, they are equivalent via
the transformation $\beta =
\frac{\alpha}{1 - \alpha}$.
(See Appendix~\ref{sec:mammen-alpha-beta}.)

Notably, in the
recent work of~\citet{MaoMoZh25}, the authors consider the exact same noise statistic
for $H$-consistency in multiclassification---while they operate under structural assumptions
of the surrogate and hypothesis class to enable a global convex lower bound, we consider
the generic setup with minimal structural constraint, which requires us to introduce the
\emph{noise condition number} in the following. 
\begin{align}
  \label{eqn:noise-condition-number}
  \kappa(x) & \defeq
  \frac{\max_{\pred(s) \neq y\opt(x)} \extloss(s, x)}{
    \min_{\pred(s) \neq y\opt(x)} \extloss(s, x)},
\end{align}
which connects the noise statistic $\Delta(x)$ and the pointwise excess risk
via $\Delta(x)\geq \extloss(s, x)/ \kappa(x)$ for all $s$ such that
$\pred(s) \neq y\opt(x)$, allowing more fine-grained analysis.
In binary classification, we have $\kappa(x) = 1$ so long as $\P(\Delta(X) >
0) = 1$.


The noise statistic $\Delta(x)$ and condition number $\kappa(x)$ will allow
us to show how (generalized) majority vote~\eqref{eqn:naive-mv-aggregation},
when applied in the context of any identifiable surrogate
(Definition~\ref{def:feasible}), achieves level-$(\xi, \zeta)$ consistency.
Define the error function
\begin{align}
  \label{eq:error-function-mv}
  e_m(t) \defeq  t\sqrt{\frac{2}{m}\log \prn{
      \frac{4k(\constant_{\sloss, 1} + \constant_{\sloss, 2})}{
        \constant_{\sloss, 1}}}},
\end{align}
which roughly captures that if $\kappa(x) = t$, then
majority vote $A_m$ is likely correct if $m$ is large enough that
$e_m(t) \ll 1$.
We then have the following theorem, which provides a (near) linear
calibration function; we prove it in
Appendix~\ref{proof:level-consistency-condition-number}.
\begin{theorem}
  \label{thm:level-consistency-condition-number}
  Let the surrogate loss $\sloss$ be $(\constant_{\sloss, 1},
  \constant_{\sloss, 2})$-identifying with parameters $\{a_y\}_{y \in \mc{Y}}$,
  and $A_m$ be the majority vote aggregator~\eqref{eqn:naive-mv-aggregation}.
  Assume the task loss satisfies $0 \le \tloss \le 1$
  and $P$ satisfies condition~\eqref{eqn:mammen-alpha}
  with noise exponent $\alpha \in [0, 1]$.  
  Then for any $M > 0$ and $f \in \mathcal{F}$ such that $R(f) - R\opt \geq
  2\Prb(\kappa(X) > M) + (4\cstMT e_m(M))^{\frac{1}{1-\alpha}}$,
  \begin{align*}
    R(f) - R\opt \leq \frac{16}{\constant_{\sloss, 1}} \cdot \prn{R_{\sloss,{A_m}}(f) - R_{\sloss,{A_m}}\opt}.
  \end{align*} 
\end{theorem}

Said differently, under the conditions of the theorem, $\sloss$ with
aggregation provides level $(\xi,\zeta)$ consistency
(Def.~\ref{definition:almost-linear-consistency}) with $\xi = 2 \P(\kappa(X)
> M) + (4 \cstMT e_m(M))^\frac{1}{1 - \alpha}$ and $\zeta =
\frac{\constant_{\sloss,1}}{16}$.
Theorem~\ref{thm:level-consistency-condition-number} also
provides an immediate proof of Corollary~\ref{corollary:always-consistent},
that is, an asymptotic guarantee of consistency.
Indeed, define
\begin{align*}
  \xi_m \defeq \inf_{M}
  \left\{ 2\Prb(\kappa(X) > M) +   (4\cstMT e_m(M))^{\frac{1}{1-\alpha}} \right\},
\end{align*}
which satisfies $\xi_m \to 0$ as $m \to \infty$,
because $\P(\kappa(X) > M) \to 0$ as $M \uparrow \infty$ and
for any fixed $M$, $e_m(M) \to 0$ as $m$ grows.
Corollary~\ref{corollary:always-consistent}
then follows trivially by taking $\alpha = 0$.

Theorem~\ref{thm:level-consistency-condition-number}
is a somewhat gross result, as the identifiability
conditions in Def.~\ref{def:feasible} are so weak.
With a tighter connection between task loss $\tloss$ and surrogate $\sloss$,
for example, making the naive majority vote~\eqref{eqn:naive-mv-aggregation}
more likely to be correct (or at least correct enough for $\sloss$),
we would expect stronger bounds---e.g., when $\varphi$ is nearly consistent
without aggregation, we may expect smaller inconsistency without aggregation
should imply stronger bounds with majority vote in $m$.
We do not pursue the details here.


To provide a somewhat more concrete bound, we optimize
over $M$ in Theorem~\ref{thm:level-consistency-condition-number},
using the crude bound $\kappa(x) \leq 1/\Delta(x)$ on the condition
number.
By taking $M=1$ for $\card(\mc{Y}) = k =2$ and optimizing $M$ for $k \geq3$,
we may lower bound $\xi_{m,k}$ in the level $(\xi_{m,k}, \zeta)$-consistency
(Def.~\ref{definition:almost-linear-consistency}) that in
Theorem~\ref{thm:level-consistency-condition-number} promises, setting
\begin{align*}
  \xi_{m,k}
  \defeq \begin{cases}
    \prn{\frac{32 \cstMT^2}{m}\log \prn{ \frac{8(\constant_{\sloss, 1}+\constant_{\sloss, 2})}{\constant_{\sloss, 1}}}}^{\frac{1}{2(1-\alpha)}},
    & \mbox{if}~ k =2
    \\
    4 \cdot  \prn{\frac{32 \cstMT^4}{m}\log \prn{ \frac{4k(\constant_{\sloss, 1}+\constant_{\sloss, 2})}{\constant_{\sloss, 1}}}}^{\frac{\alpha}{2(1-\alpha^2)}},
    & \mbox{otherwise}.
  \end{cases}
\end{align*}
Making appropriate algebraic substitutions and manipulations
(see Appendix~\ref{proof:level-consistency-optimize-M}),
we have the following corollary.
\begin{corollary}
  \label{corollary:level-consistency-optimize-M}
  Under the conditions of
  Theorem~\ref{thm:level-consistency-condition-number}, for any $f$ such
  that $R(f) - R\opt \geq \xi_{m,k}$,
  \begin{align*}
    R(f) - R\opt \leq \frac{16}{\constant_{\sloss, 1}}
    \cdot \prn{R_{\sloss,{A_m}}(f) - R_{\sloss,{A_m}}\opt}.
  \end{align*} 
\end{corollary}

The above corollary and Corollary~\ref{corollary:always-consistent} provide
evidence for the robustness of label cleaning: with minimal assumptions on
the surrogate, data aggregation can still yield consistency.
As the noise exponent $\alpha$ approaches $1$ in
Corollary~\ref{corollary:level-consistency-optimize-M}, the sample size $m$
required for the comparison inequality to hold for a fixed score function
$f$ shrinks.
Notably, if $\alpha = 1$, whenever
\begin{align*}
  m \geq 32 \max \{\cstMT^2, \cstMT^4\} \cdot \log \prn{ \frac{4k(\constant_{\sloss, 1}+\constant_{\sloss, 2})}{\constant_{\sloss, 1}}}
  = O(\log k),
\end{align*}
we have $\xi_{m,k} = 0$,
yielding the uniform comparison inequality~\eqref{item:uniform-comparison}
with linear comparison.
The noise level of the learning problem itself affects the aggregation level
needed for consistency---an ``easier'' problem requires less aggregation to
achieve stronger consistency.


%% file: surrogate-examples.tex

\subsection{Surrogate consistency examples with majority vote} \label{sec:surrogate-example}

We collect several examples, of varying levels of concreteness, that allow
us to instantiate Theorem~\ref{thm:level-consistency-condition-number} and
Corollary~\ref{corollary:level-consistency-optimize-M}.
Throughout, we shall assume that $P$ has a noise exponent $\alpha \in [0,
  1]$, though this is no loss of generality, as
Condition~\eqref{eqn:mammen-alpha} always holds with $\alpha = 0$.
\noindent
We defer proofs for each result in this section
to Appendix~\ref{sec:surrogate-example-proofs}.

\subsubsection{Binary classification with a nonsmooth surrogate}

Consider the binary classification problem with a margin-based
surrogate $\sloss(f(x), y) = \phi(yf(x))$, where $\phi$ is convex;
\citet{BartlettJoMc06} show that $\sloss$ is consistent if and
only if $\phi'(0) < 0$.
Here, we show a (somewhat trivial) example for the robustness data
aggregation offers by demonstrating that even if $\phi$ is inconsistent
without aggregation, it can become so with it.
Note, of course, that one would never \emph{use} such a surrogate, so
one ought to think of this as a thought experiment.
Assume that the subgradient set $\partial \phi(0) \subset (-\infty, 0)$ and
$\phi$ is convex with $\lim_{t \to \infty} \phi(t) = 0$.
\begin{lemma}
  \label{lemma:binary-feasible}
  For any $\delta > 0$,
  $\sloss$ is $(\constant_{\sloss, 1}, \constant_{\sloss, 2})$-feasible with
  \begin{align*}
    \constant_{\sloss, 1} = \phi(0) - \phi(\delta) > 0
    ~~ \mbox{and} ~~ \constant_{\sloss, 2} = \phi(-\delta).
  \end{align*}
\end{lemma}
Corollary~\ref{corollary:level-consistency-optimize-M} thus applies with
$k=2$,
so if $f:\sX \to \mathbb{R}$ satisfies
\begin{align*}
  & R(f) - R^\star \geq	\prn{\frac{32 \cstMT^2}{m}\log \prn{ \frac{8(\phi(-\delta) + \phi(0) - \phi(\delta))}{\phi(0) - \phi(\delta)}}}^{\frac{1}{2(1-\alpha)}},
\end{align*}
then
\begin{align*}
  R(f) - R^\star \leq \frac{16}{\phi(0) - \phi(\delta)}(R_{\sloss, A_m}(f) - R_{\sloss, A_m}^\star).
\end{align*}

\subsubsection{Bipartite matching}

In general structured prediction problems~\citep{NowozinLa11}, an embedding
map $v: \sY \to \mathbb{R}^d$ encodes structural information about
elements $y \in \sY$, where $\sY$ is some ``structured''
space, which is typically large.
Using decoder $\pred(s) = \argmax_{y \in \sY} \< s, v(y)\>$, for a loss
$\tloss : \mc{Y} \times \mc{Y} \to \R_+$ with $\tloss(y, y) = 0$, the
maximum-margin surrogate (generalized hinge loss)~\citep{TaskarGuKo03,
  TsochantaridisHoJoAl04, Joachims06} takes the form
\begin{align}
  \label{eqn:structured-prediction-loss}
  \sloss(s, y) = \max_{\hat y \in \sY}
  \prn{\tloss(\hat y, y) + \<  v(\hat y) - v(y), s \> }.
\end{align}
Notably, the loss~\eqref{eqn:structured-prediction-loss} is typically
inconsistent, except in certain low noise cases~\citep{OsokinBaLa17,
  NowakBaRu20}.

Before discussing structured prediction broadly, we consider
bipartite matching.
A bipartite matching consists of a graph $G=(V,E)$ where the vertices $V =
V_1 \cup V_2$ partition into left and right sets $V_1 = \{1, \ldots, N\}$
and $V_2 = \{N + 1, \ldots, 2N\}$, while the $N$ edges $E$ each connect
exactly one (unique) node in $V_1$ and $V_2$.
Letting $\sY$ be the collection of all bipartite matching between $V_1$ and
$V_2$, we evidently have $k = \card(\sY) = N!$.
For any graph $G$, the embedding map
\begin{equation*}
  v(G) \defeq \left(\indic{(u,v) \in E}\right)_{u \in V_1, v \in V_2} \in \R^{N^2}
\end{equation*}
indexes edges, yielding $d = N^2$.
The task loss counts the number of mistaken edges,
\begin{align*}
  \tloss(y_1,y_2) \defeq \frac{1}{2N}
  \lone{v(y_1) - v(y_2)}
  = \frac{1}{2N} \ltwo{v(y_1) - v(y_2)}^2.
\end{align*}
In this case, the max-margin (structured hinge loss)
surrogate~\eqref{eqn:structured-prediction-loss}
is identifying:
\begin{lemma}
  \label{lemma:bipartite-feasible}
  For the bipartite matching problem on $2N$ vertices, the structured
  hinge loss~\eqref{eqn:structured-prediction-loss} surrogate $\sloss$ is
  $(\constant_{\sloss, 1}, \constant_{\sloss, 2})$-identifying
  (Def.~\ref{def:feasible}) with
  \begin{align*}
    \constant_{\sloss, 1} = \frac{1}{N}
    ~~ \mbox{and} ~~
    \constant_{\sloss, 2} = 2.
  \end{align*}
\end{lemma}

The important consequence of Lemma~\ref{lemma:bipartite-feasible} is that
even when $k = \card(\sY) = N!$, aggregation-based methods can yield
consistency (via the structured hinge loss) once $m$, the number of
aggregated labels, exceeds $O(N \log N)$.
As one specialization, substituting these constants into
Corollary~\ref{corollary:level-consistency-optimize-M} for $k \geq 3$, for all
measurable $f: \sX \to \mathbb{R}^d$ such that
\begin{align*}
  R(f) - R^\star \geq 4 \cdot  \prn{\frac{32 \cstMT^4}{m}\log \prn{ 4k (2N+1)}}^{\frac{\alpha}{2(1-\alpha^2)}},
\end{align*}
one has
\begin{align*}
  R(f) - R^\star \leq 16N (R_{\sloss,A_n}(f) - R_{\sloss,A_n}^\star).
\end{align*}

\subsubsection{Structured prediction}

We return to the more general structured prediction setting, as
at the beginning of the preceding subsection.
Suppose the
decoder $\pred$ can pick any class $y \in \sY$, in that for
each $y \in \sY$, the collection
\begin{align*}
  \mc{S}(y) \defeq
  \left\{s: \< v(y), s\> > \< v(\hat y), s \>, \mbox{for~} \hat{y} \neq y\right\}
\end{align*}
of selecting $s$
is non-empty.
For each $y \in \sY$, define the identifiability gap
\begin{align*}
  \tau(y) \defeq \inf_{s \in \mathcal{S}(y)}
  \max_{y_+, y_- \neq y}
  \frac{\tloss(y_+,y)}{\<v(y) - v(y_+), s\>}
  \cdot \frac{\<v(y) - v(y_-), s\>}{
    \tloss(y_-, y)}.
\end{align*}
We have the following identifiability guarantee.
\begin{lemma}
  \label{lemma:structured-feasible}
  For the structured prediction problem,
  the max-margin~\eqref{eqn:structured-prediction-loss}
  surrogate $\sloss$ is $(\constant_{\sloss, 1}, \constant_{\sloss,
    2})$-identifiable with
  \begin{align*}
    \constant_{\sloss, 1} = \min_{\hat y \neq y} \tloss(\hat y,
    y), \qquad \constant_{\sloss, 2} = \max_{y \in \sY}\tau(y) + 1.
  \end{align*}
  In particular, if $v(y) \in \{0, 1\}^d$ and $\ell(\hat{y}, y) =
  \frac{1}{2d} \lone{v(\hat{y}) - v(y)}$, $\tau(y) = 1$ for all $y$
  and $\constant_{\sloss, 2} = 2$.
\end{lemma}

Completing the example, as in the binary matching case, we see
that nontrivial consistency guarantees hold once $m \ge \log \card(\mc{Y})$.
As $0 \leq \tloss(\cdot, \cdot) \leq 1$, 
Corollary~\ref{corollary:level-consistency-optimize-M} applies, which
yields for all measurable $f: \sX \to \mathbb{R}^d$ that
\begin{align*}
  R(f) - R^\star \geq \prn{\frac{32 \cstMT^4}{m} \log \prn{4|\sY|\left(1 + \frac{\max_{\hat y \in \sY}\tau(y) + 1}{\min_{\hat y \neq y} \tloss(\hat y, y)}\right)}}^{\frac{\alpha}{2(1-\alpha^2)}},
\end{align*}
implies
\begin{align*}
  R(f) - R^\star \leq \frac{8}{\min_{\hat y \neq y} \tloss(\hat y, y)}
  (R_{\sloss,A_m}(f) - R_{\sloss,A_m}^\star) .
\end{align*}

%% file: model-consistency.tex
\section{Robustness and consistency for models}

The previous section builds off of the now classical theory of surrogate
risk consistency, which assumes $\mathcal{F}$ to be the class of all
measurable functions.
The results there show that aggregation can allow us to ``upgrade''
consistency so that even if a surrogate $\sloss$ is inconsistent for paired
(non-aggregated) data $(X,Y)$, we can achieve level-$(\xi, \zeta)$
consistency (Def.~\ref{definition:almost-linear-consistency}) with
sufficient aggregation.
Here, we take a different view of the problem of consistency,
considering the consequences of optimizing over a restricted (often
parametric) hypothesis class $\mathcal{F}$.
Of course, in a well-specified model, obtaining consistency with such a
restricted hypothesis class is no issue, but it is unrealistic to
assume such a brittle condition.
This gives rise to the long-standing challenge of quantifying surrogate
consistency when the hypothesis class contains only a subset of the
measurable functions~\citep{DuchiKhRu16, NguyenWaJo09}. 
We point out a line of recent work develops $H$-consistency bounds that link excess target risk to surrogate regret 
within a regret hypothesis class $H$, covering various setups~\citep{AwasthiMaMoZh22,
AwasthiMaMoZh22b, MaoMoZh23, MaoMoZh24, MaoMoZh25}. These results are primarily surrogate-centric. Our contribution is orthogonal,
less hinged on the specific class and surrogate, and instead study how label aggregation improves consistency itself.
We tackle some of the issues around this, showing that
aggregating labels allows consistent estimates in scenarios where
consistency might otherwise fail.

\subsection{Consistency failure for binary classification in
  finite dimensions}
\label{sec:binary-finite-dim}

To see how restricting the hypothesis class can change the problem
substantially even in well-understood cases, we consider binary
classification.
In this case, $\mathcal{Y} = \left\{\pm 1\right\}$, and
we take the zero-one error $\tloss(\pred(s), y) = \indic{y s \le 0}$.
We consider a margin-based surrogate $\sloss(s, y) = \phi(s y)$, where
$\phi: \R \to \R_+$ is convex, and as we have discussed,
$\sloss$ achieves both Fisher~\eqref{item:pointwise-comparison}
and uniform consistency~\eqref{item:uniform-comparison}
when $\mc{F}$ consists of all measurable functions if and only if
$\phi'(0) < 0$~\citep{BartlettJoMc06}.

Now we proceed to consider a restricted hypothesis class, showing in this
simple setting that classical consistency fails even when
optimal classifiers lie in $\mc{F}$, in particular,
when $P$ is \emph{optimally predictable} using $\mc{F}$, meaning that
\begin{equation}
  \label{eqn:optimally-predictable}
  \sgn(f(x)) = \sgn(P(Y = 1 \mid X = x) - 1/2).
\end{equation}
Let $\mc{X} = \R^d$ and take $\mc{F} = \{f_\theta \mid f_\theta(x) =
\<\theta, x\>\}_{\theta \in \R^d}$ to be the collection of linear
functionals of $x$.
When $P$ is optimally predictable from using $\mc{F}$, there
exists $\theta\opt$ satisfying
$\sgn(\<\theta\opt, x\>) = \sgn(P(Y = 1 \mid x) - \half)$,
and $f_{\theta\opt}$ minimizes $R(f)$ across all measurable functions.
In this case, we say that $P$ is \emph{optimally predictable along
$\theta\opt$}.
One might expect a margin-based surrogate $\varphi$ achieving Fisher
consistency in the classical setup should still consistent.
This fails.
Even more, for any nonnegative loss $\phi$, there is a data distribution $P$
on $(X, Y)$ such that $\theta_\sloss = \argmin_\theta
\E_P[\sloss(f_\theta(X), Y)]$ is essentially orthogonal to $\theta\opt$:

\begin{proposition}
  \label{proposition:failure-binary}
  For any $\epsilon > 0$ and nonzero vector $\theta\opt \in \R^d$, there
  exists an $(X,Y)$ distribution $P$, optimally
  predictable
  along $\theta\opt$, such that for all
  \begin{align*}
    \theta_\sloss \in \argmin_{\theta} R_\sloss (f_\theta)
    = \argmin_{\theta} \E\left[\phi(f_\theta(X), Y)\right],
  \end{align*}
  we have $R(f_{\theta_\sloss}) > R(f_{\theta\opt})$ and $|\cos \angle
    (\theta_\sloss, \theta\opt )| = |\<\theta_\sloss, \theta\opt\>|/
    (\ltwo{\theta_\sloss} \ltwo{\theta\opt}) \leq \epsilon$.
\end{proposition}
\noindent
We postpone the proof to Appendix~\ref{proof:failure-binary}.

Data aggregation methods provide one way to circumvent the inconsistency
Proposition~\ref{proposition:failure-binary} highlights.
To state the result, define the approximate minimizers
\begin{equation*}
  \epsargmin g = \epsargmin_\theta g(\theta) \defeq
  \Big\{\theta \mid g(\theta) \le
  \inf_\theta g(\theta) + \epsilon \Big\}.
\end{equation*}
Suppose the data collection consists of independent samples $Z=(Y_1, \ldots,
Y_m)$ and we take $A_m(Z)$ to be majority vote.
For a sequence $\epsilon_m$ take
\begin{align*}
  \theta_m \in \epsargmin[m]_{\theta}
  R_{\sloss,A_m}(f_\theta)
  = \epsargmin[m]_{\theta} \E \brk{\phi(A_m(Y_1, \ldots, Y_m) f_\theta(X))}.
\end{align*}
Then as a corollary to the coming Theorem~\ref{thm:restricted-class},
$f_{\theta_m}$ are asymptotically consistent when $m \to \infty$.

\begin{corollary}
  \label{cor:agg-consistency-binary}
  Let $P$ be optimally predictable along $\theta\opt$.
  Then if $\epsilon_m \to 0$ as $m \to \infty$,
  $R(f_{\theta_m}) \to R(f_{\theta\opt})$
  and $\cos \angle (\theta_m, \theta\opt ) \to 1$.
\end{corollary}
\noindent
So without aggregation, surrogate risk minimization is (by
Proposition~\ref{proposition:failure-binary}) essentially arbitrarily
incorrect when restricting to the class of linear predictors, while with
aggregation, we retain consistency.

\subsection{Aggregation, consistency, and restricted hypothesis classes}

As Proposition~\ref{proposition:failure-binary} shows,
surrogate risk consistency reposes quite fundamentally on
$\mc{F}$ containing all measurable functions.
We now consider multiclass classification problems, where
$\mc{Y} = \{1, \ldots, k\}$, and in which $\mc{F}$ forms a
linear cone satisfying
\begin{equation*}
  f(x)^T \ones = 0
  ~~ \mbox{and} ~~ t f \in \mc{F}
  ~ \mbox{for~} t > 0 ~ \mbox{if}~ f \in \mc{F}.
\end{equation*}
We consider the zero-one loss $\tloss(y, y') = \indic{y \neq y'}$ and
$\pred(s) = \argmax_{y \in [k]} s_y$, making the restriction to predictors
normalized to have $f(x)^T \ones = 0$ immaterial.
Assume the surrogate $\sloss : \R^k \times [k] \to \R_+$ is
Fisher consistent~\eqref{item:pointwise-comparison} and satisfies the
limiting loss condition
\begin{equation}
  \label{eqn:multiclass-surrogate-limit}
  \sloss(s, y) \to 0
  ~~ \mbox{if} ~~
  s_y - s_j \to +\infty
  ~ \mbox{for~all}~ j \neq y.
\end{equation}

Many familiar surrogate losses are Fisher
consistent and satisfy~\eqref{eqn:multiclass-surrogate-limit}, including the
multiclass logistic loss $\sloss(s, y) = \log(\sum_{j = 1}^k e^{s_j - s_y})$
and any loss of the form
\begin{align*} 
  \sloss(s, y) = \sum_{i \neq y} \phi(s_y - s_i)
\end{align*}
for $\phi$ convex, non-increasing with $\phi'(0) < 0$,
and $\inf_t \phi(t) = 0$.
\citet[Thm.~8]{Zhang04a} shows that any such loss is
consistent over the class $\mc{F} = \{f : \mc{X} \to \R^k \mid \ones^T f = 0\}$.
Clearly, the margin-based binary setting in Sec.~\ref{sec:binary-finite-dim}
falls into this scenario when we take $f(x) = (g(x), -g(x))$ for a
measurable $g$.
Additionally, in a parametric setting when $\sX = \R^d$, if $\mc{F}$
consists of linear functions $f(x) = (\<\theta_1, x\>, \ldots, \<\theta_k,
x\>)$ with $\sum_{i = 1}^k \theta_i = 0$, then $\mc{F}$ is a (convex) cone. However, those examples typically require strong structural assumptions such as convexity for consistency without aggregation. As we shall see next, with the minimal limiting loss condition~\eqref{eqn:multiclass-surrogate-limit}, we can achieve asymptotic consistency by aggregation.

Extending the definition~\eqref{eqn:optimally-predictable}
of optimal predictability in the obvious way,
we shall say $\mc{F}$ can optimally predict $P$ if there exists
$f \in \mc{F}$, $f : \mc{X} \to \R^k$, for which
\begin{equation*}
  \argmax_y f_y(x) \in
  \argmax_y P(Y = y \mid X = x)
  ~~ \mbox{for~all~} x.
\end{equation*}
The next theorem shows if $\mathcal{Z} = \mathcal{Y}^m$, we aggregate via
majority vote $A_m$, and there is a unique $y\opt(x) = \argmax_y P(Y=y \mid
x)$, then surrogate risk minimization is consistent whenever $\mc{F}$ can
optimally predict $P$.
\begin{theorem}
  \label{thm:restricted-class}
  Let $\mc{F}$ be a cone that optimally predicts $P$,
  and assume that the
  minimal excess risk~\eqref{eqn:minimal-wrong-excess-risk}
  satisfies $P(\Delta(X) > 0) = 1$.
  Let $\epsilon_m \ge 0$ satisfy $\epsilon_m \to 0$.
  Then for any sequence
  \begin{align*}
    f_m \in \epsargmin[m]_{f \in \mathcal{F}}
    R_{\sloss,A_m}(f) = \epsargmin[m]_{f \in \mc{F}}
    \E\brk{\sloss(f(X), A_m(Y_1, \ldots, Y_m))},
  \end{align*}
  we have $R(f_m) \to R\opt$.
\end{theorem}
\noindent
See Appendix~\ref{proof:restricted-class} for a proof.

Theorem~\ref{thm:restricted-class} shows that for a broad class of surrogate
problems with a hypothesis class $\mathcal{F}$ that forms a linear cone, we
can achieve consistency asymptotically by aggregation as $m \to \infty$.
In contrast, as Proposition~\ref{proposition:failure-binary} shows, even in the
``simple'' case of binary classification, consistency may fail over
subclasses $\mc{F}$, even when they include the optimal predictor, and the
surrogate can be arbitrarily uninformative.

\subsection{On finite-dimensional multiclass classification}
\label{sec:finite-dim}

The final technical content of this paper
considers a multiclass scenario in which $\mc{X} = \R^d$, $\mc{Y} = [k]$,
and we use linear predictors,
but where the predictive model may be mis-specified.
These results will provide a more nuanced and specialized view
of the situation than the negative results in
Proposition~\ref{proposition:failure-binary} and the consistency guarantee
in Theorem~\ref{thm:restricted-class}.
We will show that even when the problem is optimally predictable and the
linear model is well-specified on all except an $\epsilon$-fraction of
data, surrogate risks based only on $(X, Y)$ pairs are inconsistent;
majority vote-based methods, however, will recover the optimal linear
predictor.

To set the stage, let
$\Theta = [\theta_1 ~ \cdots ~ \theta_{k-1}] \in \R^{d \times k-1}$,
and let the labels follow a categorical distribution $Y
\mid X = x\sim \cat(p_\Theta(x))$, where $p_\Theta(x) \in \R^k_+$ satisfies
$\ones^\top p_\Theta(x) = 1$ and
\begin{align*}
  p_\Theta(x) = \begin{bmatrix}
    \sigma(\< \theta_1 , x \>, \cdots, \< \theta_{k-1} , x \>) \\
    1 - \ones^\top \sigma(\< \theta_1 , x \>, \cdots, \< \theta_{k-1} , x \>)
  \end{bmatrix}
\end{align*}
for a link $\sigma : \R^{k-1} \to \R^{k-1}_+$,
$\ones^\top \sigma \le 1$.
We assume $\sigma$ satisfies the consistency condition that for $t \in
\R^{k-1}$, setting $t_k = 0$ and $\sigma_k(t) = 1 - \ones^\top \sigma(t)$,
\begin{equation}
  \label{eqn:consistent-link}
  t_y = \max_{1 \leq i \leq k} t_i
  ~~ \mbox{if and only if} ~~
  \sigma_y(t) = \max_{1 \leq i \leq k} \sigma_i(t). 
\end{equation}
One standard example is multiclass logistic regression, where
\begin{align*}
  \sigma_i\lr(t) & = \frac{e^{t_i}}{1 + e^{t_1} + \cdots + e^{t_{k-1}}}.
\end{align*}
Let the multilabeled dataset be $\{(X_i; (Y_{1i}, \cdots,
Y_{im}))\}_{i=1}^n$ with repeated sampling $Y_{ij} \mid X_i \simiid
\cat(p_{\Theta^\star}(X_i))$, and $Y_i^+ = A_m(Y_{i1}, \cdots, Y_{im})$ be
the majority vote with ties broken arbitrarily.
Then $Y_i^+ \mid X_i \sim \cat(p_{\Theta^\star,m}(X_i))$, where if
$\rho_m(t)$ denotes the distribution of majority vote on $m$ items with
initial probabilities $\sigma(t) \in \R^{k-1}_+$, then
\begin{align*}
  p_{\Theta, m}(x) = \begin{bmatrix}
    \rho_m(\< \theta_1 , x \>, \cdots, \< \theta_{k-1} , x \>) \\
    1 - \ones^\top \rho_m(\< \theta_1 , x \>, \cdots, \< \theta_{k-1} , x \>)
  \end{bmatrix}.
\end{align*}
It is evident that $\rho_m$ satisfies link
consistency~\eqref{eqn:consistent-link}.
Consider fitting a logistic regression with loss
\begin{align}
  \sloss(\Theta^\top x, y)
  =- \< \theta_y, x \> + \log \prn{1 + \sum_{i=1}^{k-1} \exp(\<\theta_i, x\>)}, \label{eq:logistic-loss}
\end{align}
with the convention that $\theta_k = \zeros$,
and let $L_m(\Theta) = \E[\sloss(\Theta^\top X, Y_m^+)]$ be the
logistic loss with $m$-majority vote.
Then
\begin{align*}
  \nabla_{\theta_i} \sloss(\Theta^\top x, y)
  = -x (\indic{y=i} - \sigma_i\lr(\Theta^\top x)),
\end{align*}
so that the $\Theta_m$ minimizing $L_m$ satisfies
\begin{align*}
  \nabla_{\Theta} L_m(\Theta_m)
  = \E \brk{X \prn{\sigma\lr(\Theta_m^\top X)
      - \rho_m({\Theta\opt}^\top X)}^\top}
  = 0.
\end{align*}

Standard results in statistics guarantees both consistency and efficiency
when the model is well-specified without aggregation, and when $m \ge 1$ and
$k = 2$, \citet[Prop.~3]{ChengAsDu22} show there exists $t_m > 0$ such that
$\Theta_m = t_m \Theta^\star$
if $X \sim \normal(0, I_d)$ even with a mis-specified link.
This implies that in binary classification, even if the link function is
incorrect, we can still achieve consistent classification regardless of
the aggregation level $m$, as the direction $\Theta\opt / \ltwo{\Theta\opt}$
determines consistency.
However, as soon as $k \geq 3$ and the true link is slightly
mis-specified, risk consistency fails.
Fixing a set $\mc{T}_\epsilon \subset \R^{k-1}$ with Lebesgue
measure $\epsilon$, consider
\begin{align*}
  \sigma\lre (t) = \sigma\lr (t) \ind\{t \not \in \mc{T}_\epsilon\} +
  \frac{1}{k} \ones \cdot \indic{t \in \mc{T}_\epsilon},
\end{align*}
which defines a distribution on $Y \in \{1, \ldots, k\}$
conditional on $t \in \R^{k-1}$ that samples
$Y \sim \sigma\lr(t)$ if $t \in \mc{T}_\epsilon$ and
uniformly otherwise.
Clearly $\sigma\lre$ satisfies the link consistency
condition~\eqref{eqn:consistent-link} and is optimally
predictable~\eqref{eqn:optimally-predictable}.
For $\epsilon > 0$, define
\begin{equation*}
  L_{m,\epsilon}(\Theta)
  \defeq \E[\E_{\sigma\lre}[\sloss(\Theta^\top X, Y_m^+) \mid X]]
\end{equation*}
to be the (population) logistic loss, based on $m$-majority vote,
when $Y \mid X = x \sim \sigma\lre({\Theta\opt}^\top x)$.
Let $\Theta_m(\epsilon) = \argmin_\Theta L_{m,\epsilon}(\Theta)$.
Evidently, $\Theta_1(0) = \Theta\opt$; nonetheless, the next result shows
that for arbitrarily small $\epsilon > 0$, consistency fails without
aggregation.
\begin{proposition}
  \label{proposition:consistency-failure-multiindex}
  Let $k \geq 3$, $\Sigma = I$.
  Assume that for $Z \sim \normal(0, I_{k-1})$, the linear mapping $M
  \mapsto D M \defeq \E[Z (\nabla\sigma\lr({\Theta\opt}^\top Z) M Z)^\top]$ is
  invertible.
  Then there exists $\epsilon_0 > 0$ such that for any $\epsilon \in (0,
  \epsilon_0)$, there is a set $\mc{T}_\epsilon$ with
  Lebesgue measure at most $\epsilon$ and for which
  \begin{equation*}
    \Theta_1(\epsilon) / \norm{\Theta_1(\epsilon)} \neq
    \Theta^\star / \norm{\Theta^\star}.
  \end{equation*}
\end{proposition}
\noindent
We postpone the proof to
Appendix~\ref{sec:proof-consistency-failure-multiindex}.

Majority vote, however, can address this inconsistency as $m \to \infty$
without any assumptions on the true link $\sigma$ except that it satisfies
the consistency condition~\eqref{eqn:consistent-link}.
Indeed, letting $L_{m,\sigma}(\Theta) = \E[\E_\sigma[\sloss(\Theta^\top X,
    Y_m^+) \mid X]]$ and $\Theta_m = \argmin L_{m,\sigma}(\Theta)$, we have
the following
\begin{theorem}
  \label{thm:maj-multiindex}
  Let $\Theta\opt$ have decomposition $\Theta\opt = U\opt T\opt$, where $U\opt
  \in \R^{d \times (k-1)}$ is orthogonal and $T\opt \in \R^{(k-1) \times
    (k-1)}$ is nonsingular.
  Then there exists $T_m \in \R^{(k-1) \times (k-1)}$ such that $\Theta_m =
  U^\star T_m$ for every $m$, and as $m \to \infty$,
  \begin{align*}
    \|T_m\| \to \infty
    ~~ \mbox{and} ~~ T_m / \|T_m\| \to T^\star / \|T^\star\|.
  \end{align*}
\end{theorem}
\noindent
See Appendix~\ref{proof:maj-multiindex} for a proof.

Theorem~\ref{thm:maj-multiindex} shows more evidence for the robustness
properties of label aggregation, providing asymptotic consistency even in
mis-specified models so long as there is \emph{some} link function
describing the relationship between $X$ and $Y$.
The robustness is striking when $k \geq 3$: as
Proposition~\ref{proposition:consistency-failure-multiindex} highlights,
methods without label aggregation are generally inconsistent.

%% file: numerical.tex
\section{Numerical illustrations}
In this section, we provide two experiments that corroborate the theory: (i) a non-parametric example where aggregation amplifies surrogate consistency, and (ii) a finite-dimensional multiclass example where aggregation recovers asymptotic consistency under link mis-specification.

\paragraph{Surrogate consistency amplification for truncated quadratic loss.} In binary classification, consider the classical margin-based surrogate $\sloss(f(x), y) = \phi(yf(x))$ where $\phi(\delta) = \max\{1-\delta, 0\}^2$ is the truncated quadratic loss. By~\citet[Example 2]{BartlettJoMc06}, the calibration function $\psi$ in Corollary~\ref{cor:comparison-inequality-classical} is exactly $\psi(\epsilon) = \epsilon^2$. Leveraging label aggregation through majority vote with $m$ labels, we can apply Lemma~\ref{lemma:binary-feasible} taking $\delta = 1$. In the regime with no low-noise condition (i.e. $\alpha = 0, \cstMT = 1$), taking
\begin{align*}
\psi_m(\epsilon) & = \frac{\phi(0) - \phi(\delta)}{16} \left(\epsilon - \sqrt{\frac{32}{m}\log \prn{ \frac{8(\phi(-\delta) + \phi(0) - \phi(\delta))}{\phi(0) - \phi(\delta)}}}\right)_+ = \frac{1}{16} \left(\epsilon-\sqrt{\frac{32}{m} \log 40}\right)_+,
\end{align*}
we have parallel comparison inequalities for non-aggregated data and majority-vote aggregation as
\begin{align*}
	 \psi(R(f) - R\opt) \leq R_\sloss(f)  - R_\sloss\opt, \qquad   \psi_m(R(f) - R\opt) \leq R_{\sloss,A_m}(f) - R_{\sloss,A_m}\opt.
\end{align*}
Numerically (cf.~Fig.~\ref{fig:comparison}), the aggregated curve transitions from quadratic toward linear behavior as $m \to \infty$, showing label aggregation upgrades surrogate consistency. Nonetheless, $\psi_m$ are generally conservative as they are from the general result in Theorem~\ref{thm:level-consistency-condition-number}.
\begin{figure}[htbp!]
	\centering
	\includegraphics[width=.75\linewidth]{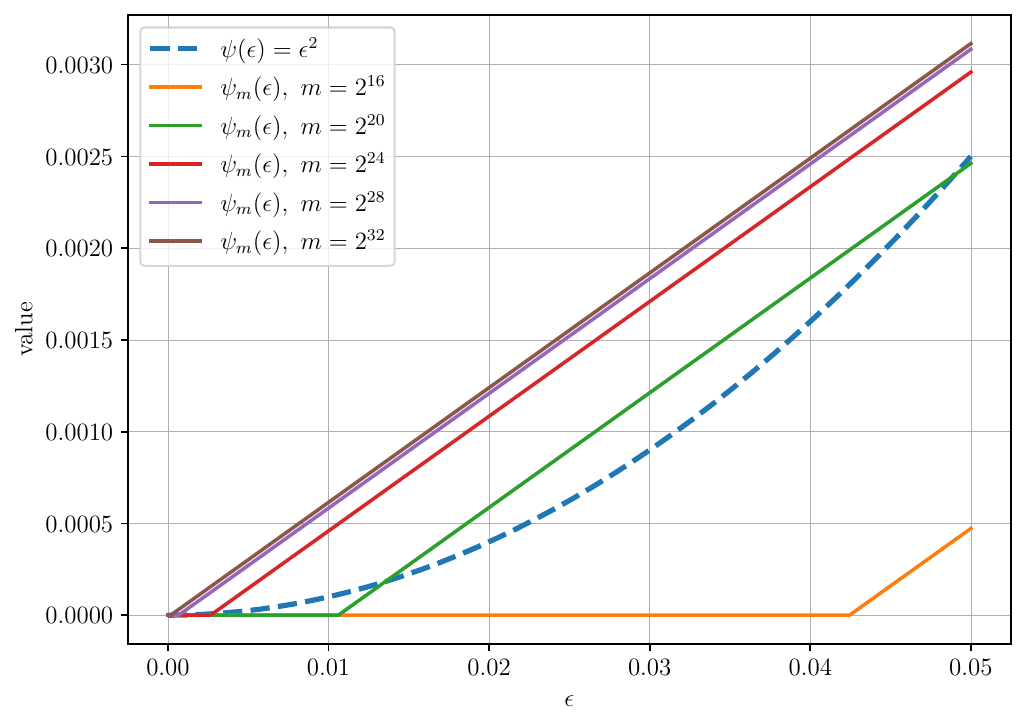}
	\caption{Numerical illustrations for truncated quadratic surrogate $\phi(\delta) = \max\{1-\delta, 0\}^2$, showing comparison inequalities for non-aggregated data $\psi$ (optimal) vs. majority vote aggregation $\psi_m$ (suboptimal) when $m \in \{2^{16}, 2^{20}, 2^{24}, 2^{28}, 2^{32}\}$.} \label{fig:comparison}
\end{figure}

\paragraph{Asymptotic consistency for mis-specified logistic model.} In the second example, we consider logistic regression in Sec.~\ref{sec:finite-dim} when $k=3$ and $d=2$. Let the loss function $\varphi$ be the logistic loss in Eq.~\eqref{eq:logistic-loss}, with a corrupted logistic link defined as
\begin{align*}
	\tilde{\sigma}(t) & = \sigma\lr(t) \ind \{ \norm{t} \leq r\} + \sigma\lr(\alpha t)\ind \{ \norm{t} > r\},
\end{align*}
a piecewise link that matches softmax near the origin and shrinks scores outside a radius $r$. The numerical results are in Fig.~\ref{fig:logistic}, showing asymptotic consistency as $m \to \infty$ with synthetic dataset from the corrupted link.
\begin{figure}[htbp!]
	\centering
    \begin{tabular}{cc}
	\includegraphics[width=.45\linewidth]{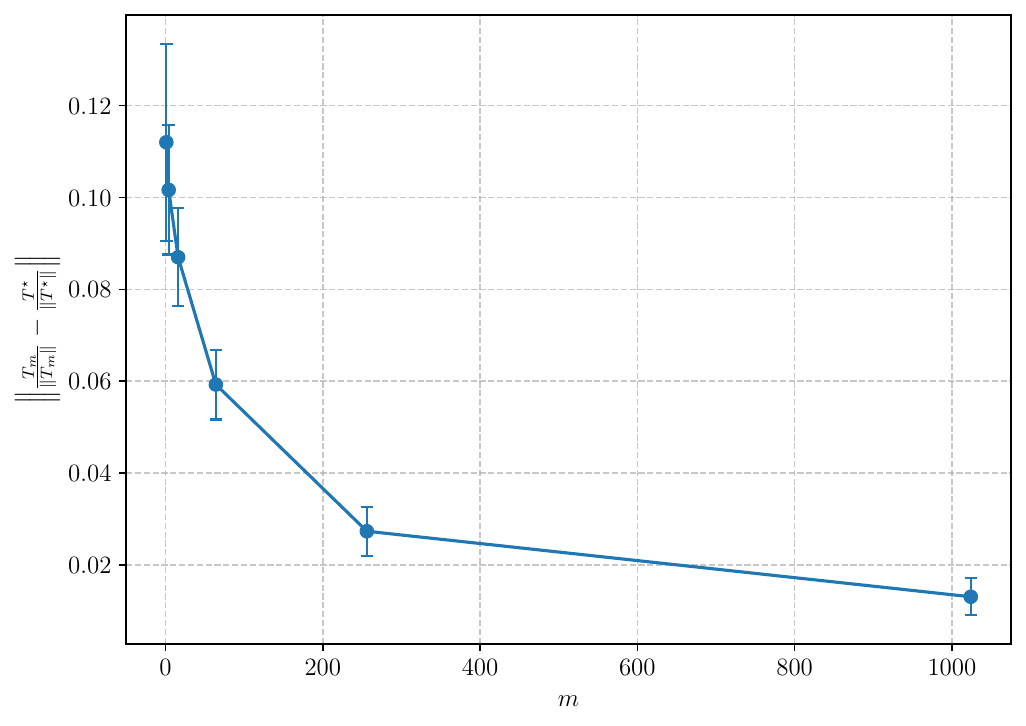} &
    \includegraphics[width=.45\linewidth]{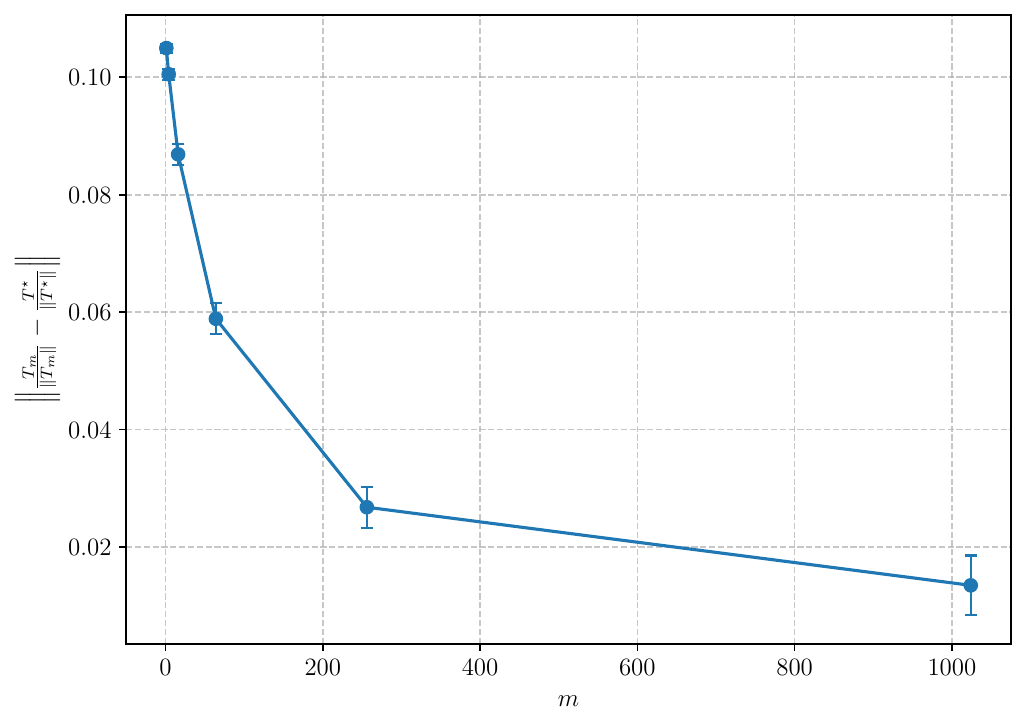} \\
    (i) Experiment with constant dataset size $n$ & (ii) Experiment with constant label size $nm$.
    \end{tabular}
	\caption{Numerical illustrations for mis-specified logistic models when $d=2$ and $k=3$. The true parameters are $\theta_1^\star = (0, 1)^\top, \theta_2^\star = (1, 1)^\top$. The corrupted link parameters are $r = 3$ and $\alpha = 0.1$. The plot shows the error of classification rule $\left\|\|T_m\| - T^\star / \|T^\star\|\right\|$ vs. number of labels $m \in \{1, 4, 16, 64, 256, 1024\}$. (i) In the left panel, for each value of $m$, the synthetic dataset consists of $n = 10{,}000$ multilabeled observations $(X_i, (Y_{i1}, \ldots, Y_{im}))$ sampled from $\tilde{\sigma}$. (ii) In the right panel, for each $m$, we set $n = 640{,}000/m$ so that the total number of labels remains constant across different values of $m$. We report a $95\%$ error bar on $T=100$ trials for each $m$.} \label{fig:logistic}
\end{figure}

%% file: discussion.tex
\section{Discussion and future work}

The question of whether and how to clean data has
animated much of the research discussion around dataset collection.
\citet*{ChengAsDu22} provide a discussion of these issues, highlighting that
there appears to be a phenomenon that using non-aggregated data---all
available labels---leads to better statistical efficiency when models have
the power to fully represent all uncertainty, but otherwise, data cleaning
appears to be more robust.
In a similar vein, \citet{DornerHa24} argue that, in a validation setting of
comparing binary classifiers, it is better to use more noisy labels rather
than cleaned variants.
This paper contributes to this dialogue by providing evidence for both
fundamental limits to using un-cleaned, un-aggregated label information
in supervised learning while highlighting robustness improvements
that come from label cleaning.
Nonetheless, many questions remain.

\paragraph{Finite $m$ results and fundamental limits.}
Many of the consistency results we present repose on taking a limit
as $m$, the number of labels aggregated, tends to infinity.
While at some level, the purpose of this paper is to highlight ways
in which label aggregation can improve robustness, it is
perhaps unsatisfying to rely on this asymptotic setting.
In the context of ranking (Sec.~\ref{sec:counter-example}), we
can provide explicit consistency guarantees at a finite $m$,
but developing this further provides one of the most
natural and (we believe) important avenues for future work.
Providing a surrogate consistency theory that depends both on
the loss pairs $(\tloss, \sloss)$ and the available label
count $m$ would be interesting; for example,
in the context of ranking in Sec.~\ref{sec:counter-example},
if we wish to look at second or third-order comparisons
of items (e.g., powers $C_x^p$, as \citet{Keener93} suggests),
do we require increasing label counts $m$?
Precisely delineating those problems that require label cleaning and
aggregation from those that do not represents a central challenge here.

\paragraph{Fundamental limits of the noise condition number.}
Our work relies on the noise condition number~\eqref{eqn:noise-condition-number},
$\kappa(X)$, to characterize comparison inequalities for
multiclassification problems, hinting at the difficulty beyond
binary classification, where trivially $\kappa(X) = 1$.
The condition number can still be large even when the Mammen-Tsybakov noise
level~\eqref{eqn:mammen-alpha} is low---i.e.\ $\alpha \approx 1$---in cases
beyond binary classification.
This is a consequence of the
minimal assumptions on the surrogate in our setting, and it would
be beneficial to identify
connections between the loss $\tloss$ and surrogate $\sloss$
that more closely capture problem difficulty.
A more precise delineation of
fundamental limits by constructing explicit failure modes will
also yield more insights into fitting predictive models.

\paragraph{Behavior in mis-specified models.}
Our results on mis-specified models, especially those in
Section~\ref{sec:finite-dim}, require optimal
predictability~\eqref{eqn:optimally-predictable}, that is, that a
Bayes-optimal classifier lie in $\mc{F}$.
While classical surrogate consistency results provably fail even in this
case---and methods based on aggregated labels can evidently succeed---moving
beyond such restricted scenarios seems a fruitful and interesting direction.
\citet{NguyenWaJo09}, followed by \citet{DuchiKhRu16}, identify one
direction here, showing that in binary and multiclass classification
(respectively), jointly inferring a predictor $f$ and a data representation
for $x$ requires that surrogates $\sloss$ take a particular form depending
on the task loss $\tloss$.
These still repose on infinitely powerful decision rules $f$, however,
so we need new approaches.

\paragraph{Restricted class comparison inequality with label aggregation.}
Our paper investigates how label aggregation mitigates regret induced by surrogate losses
through comparison inequalities in the non-parametric setting, and establishes statistical consistency in the parametric regime. Motivated by recent developments in $H$-consistency~\citep{AwasthiMaMoZh22, MaoMoZh24, MaoMoZh24b, MaoMoZh25}, it would be intriguing to extend these techniques to analyze comparison inequalities under restricted hypothesis classes. In particular, \citet{MaoMoZh24b} shows that for smooth convex loss, quadratic growth of $\psi$ is rate optimal---it will be interesting to apply aggregation
and upgrade consistency in the context of $\mc{H}$-consistency. Such an extension could offer a unified framework linking the advantages of aggregation across parametric and non-parametric settings, while shedding light on how the ``size'' of the hypothesis class 
$\mc{F}$ influences the effectiveness of aggregation.



%% file: proof-ranking.tex
\section{Proofs related to the ranking examples
  (Sec.~\ref{sec:counter-example})}
\label{sec:proofs-counter-example}

\subsection{Proof of Proposition~\ref{proposition:failure-ranking}}
\label{proof:failure-ranking}

The proof relies on a few notions of variational convergence of
functions~\citep{RockafellarWe98}, which we review presently.
Recall that
for a sequence of sets $A_n \subset \R^k$,
\begin{equation*}
  \limsup_n A_n \defeq \left\{x \in \R^k \mid
  \liminf_n \dist(x, A_n) = 0 \right\}
  = \left\{x \mid \mbox{there~are~} y_n \in A_n
  ~ \mbox{s.t.}~ y_{n(m)} \to x \right\}
\end{equation*}
and that for a function $g$, we define
\begin{equation*}
  \epsargmin g = \left\{s \mid g(s) \le \inf g + \epsilon \right\}
\end{equation*}
It will be important for us to discuss convergence of minimizers of convex
functions, and to that end, we state the following
consequence of the results in \citet{RockafellarWe98},
where $\wb{\R} = \R \cup \{+\infty\}$.
\begin{lemma}
  \label{lemma:epi-convergence}
  Let $g_n : \R^k \to \wb{\R}$ be convex functions with pointwise limit
  $g$ where $g$ is coercive. Then $g_n$ converges uniformly to $g$ on
  compacta, $g$ is convex, the $g_n$ are eventually coercive,
  and for any sequence $\epsilon_n \downarrow 0$
  (including those with $\epsilon_n = 0$),
  \begin{equation*}
    \emptyset \neq
    \limsup_n \left\{\epsargmin[n] g_n\right\}
    \subset \argmin g.
  \end{equation*}
\end{lemma}
\begin{proof}
  First, we observe that $g_n \to g$ pointwise implies that $g_n \to g$
  uniformly on compacta, and $g$ is convex
  (see~\citet[Thm.~IV.3.15]{HiriartUrrutyLe93}
  and~\citet[Thm.~7.17]{RockafellarWe98}). This is then equivalent to
  epigraphical convergence of $g_n$ to
  $g$~\citep[Thm.~7.17]{RockafellarWe98}. Moreover, as $g_n \to g$ uniformly
  on compacta, if $x_n \to x$ then $g_n(x_n) \to g(x)$. Thus,
  for any $\epsilon_n \downarrow 0$, if for a subsequence
  $n(m) \subset \N$ we have $x_{n(m)} \in \epsargmin[n(m)] g_n$,
  and $x_{n(m)} \to x$, we certainly have
  $g_{n(m)}(x_{n(m)}) \to g(x)$.
  Consequently~\citep[Prop.~7.30]{RockafellarWe98} we have
  \begin{equation*}
    \limsup_n \left\{\epsargmin[n] g_n\right\} \subset \argmin g.
  \end{equation*}
  That the limit supremum is non-empty is then a consequence
  of~\citet[Thm.~7.33]{RockafellarWe98}, as the convex functions
  $g_n$ must be coercive as they are convex and $g$ is.
\end{proof}

We now outline our approach and leverage a few consequences
of Lemma~\ref{lemma:epi-convergence}.
Recall our restriction of $\sloss$ to the set $s^T \ones = 0$.
For probabilities
$p = (p_{ij})_{i,j \le k}$, define
\begin{equation*}
  R_\sloss(s \mid p) \defeq
  \E_{Y \sim p}\left[\sloss(s, Y)\right].
\end{equation*}
We argue that for appropriate $p$, $R_\sloss$ is coercive, and then use
Lemma~\ref{lemma:epi-convergence} to argue about the structure of its
minimizers.
Assume for the sake of contradiction that $\sloss$ is
consistent.
By considering a distribution $p$ supported only on the pair $(i, j)$,
appealing to standard results on surrogate risk consistency for binary
decision problems~\citep{BartlettJoMc06} shows that $\sloss(s, (i, j)) \to
\infty$ whenever $(s_j - s_i) \to \infty$.
Now, consider any distribution $p$ containing a
cycle over all $i \in \{1, \ldots, k\}$, meaning that there exists a
permutation $\pi : [k] \to [k]$ such that $p_{\pi(i),\pi(i+1)} > 0$ for all
$i$ (where $\pi(k+1) = \pi(1)$).
Then
\begin{equation*}
  R_\sloss(s \mid p) \ge \min_{i \in [k]} p_{\pi(i), \pi(i + 1)}
  \sloss(s, (\pi(i), \pi(i + 1))),
\end{equation*}
and without loss of generality, we assume $\pi(i) = i$.
If $\norm{s} \to \infty$ while $\ones^T s = 0$ (recall the assumption in the
proposition), it must be the case that $\max_i (s_{i + 1} - s_i) \to
\infty$, and so $s \mapsto R_\sloss(s \mid p)$ is coercive whenever $p$
contains a cycle and the minimizers of $R_\sloss(\cdot \mid p)$ exist.

With these preliminaries, we turn to the proposition proper.
We construct a distribution $p \in \R^{k \times k}_+$ for which $\zeros$
must be a minimizer of $R_\sloss(s \mid p)$, and use this to show that
$\zeros$ minimizes $R_\sloss(s \mid (i, j))$ for each pair, yielding a
contradiction to Fisher consistency.
Consider a distribution $p \in \R_+^{k \times k}$
parameterized by $q \in \R^k_{++}$, i.e., $q > 0$,
satisfying $\ones^T q = 1$.
Then define $p$ to have entries
\begin{equation}
  \label{eqn:construction-p-ranking}
  p_{ij} = \begin{cases}
    q_l, & \mbox{if}~ (i,j) = (l, l+1), \\
    q_k = 1 - q_1 - \cdots - q_{k-1} > 0, & \mbox{if~}
    (i,j) = (k,1), \\
    0, & \text{otherwise}.
  \end{cases}
\end{equation}
The corresponding normalized transition matrix $C \defeq C(q_1,\ldots,q_k)$
then takes the form
\begin{align*}
  C_{ij} = \begin{cases}
    1, & j=i+1 \text{ or }(i,j) = (k, 1), \\
    0, & \text{otherwise}.
  \end{cases}
\end{align*}
Evidently $C \ones = \ones$.

If $\sloss$ is Fisher
consistent, we claim that $\zeros$ must minimize the conditional
surrogate risk
\begin{align}
  R_\sloss(s \mid p)
  = \E_{Y \sim p} \brk{\sloss(s, Y)} = \sum_{l=1}^k q_l \sloss(s, (l,l+1)).
  \label{eq:ranking-conditional-surrogate-risk}
\end{align}
To see this, fix an (arbitrary) permutation $\pi$.
We tacitly construct a sequence $p\sups{n} \to p$ with $p\sups{n} \in \R^{k
  \times k}_+$, ${p\sups{n}}^T \ones = \ones$,
for which the comparison matrix
$C\sups{n}$ with non-diagonal entries
\begin{equation*}
  C\sups{n}_{ij} = \frac{p\sups{n}_{ij}}{\sum_{l \neq i} p\sups{n}_{lj}}
\end{equation*}
satisfies $[C\sups{n}\ones]_{\pi(i)} > [C\sups{n}\ones]_{\pi(i + 1)}$ for
each $i$.
(To perform this construction, take scalars $v_1 > \cdots > v_k > 0$,
and add $\frac{1}{n} v_i$ to each entry of row $\pi(i)$ in $p$,
so that if $v_{\pi^{-1}} = (v_{\pi^{-1}(1)}, \ldots, v_{\pi^{-1}(k)})$,
then $p\sups{n} = (p + \ones v_{\pi^{-1}}^T / n) /
\ones^T (p + \ones v_{\pi^{-1}}^T / n) \ones$.
Let $n$ be large.)

The presumed Fisher consistency of $\sloss$ means it must
be the case that
\begin{equation*}
  s\sups{n} \in \argmin_s R_\sloss(s \mid p\sups{n})
  ~~ \mbox{satisfies} ~~
  s\sups{n}_{\pi(i)} > s\sups{\pi,n}_{\pi(i+1)}
  ~~ \mbox{for~each}~i.
\end{equation*}
Applying Lemma~\ref{lemma:epi-convergence} for each such sequence
and permutation $\pi$, we see that
the set of minimizers $\argmin_s R_\sloss(s \mid p)$ of the
conditional risk~\eqref{eq:ranking-conditional-surrogate-risk}
must, for each permutation $\pi$, include a vector $s = s(\pi)$ such that
\begin{equation}
  \label{eqn:all-permutation-orderings}
  s_{\pi(1)} \ge s_{\pi(2)} \ge \cdots \ge s_{\pi(k)}.
\end{equation}
As $\argmin_s R_\sloss(s \mid p)$ is a closed convex set, we now apply
the following
\begin{lemma}
  Let $S \subset \{s \in \R^k \mid s^T\ones = 0\}$ be a convex set containing
  a vector of the form~\eqref{eqn:all-permutation-orderings} for each
  permutation $\pi$. Then $\zeros \in S$.
\end{lemma}
\begin{proof}
  We proceed by induction on $k \ge 2$.
  Certainly for $k = 2$, given
  vectors $u = (s, -s)$ and $v = (-t, t)$ satisfying
  $s \ge 0$ and $t \ge 0$, we solve
  \begin{equation*}
    \lambda s + (1 - \lambda) t = 0
    ~~ \mbox{or} ~~
    \lambda = \frac{t}{s + t} \in [0, 1],
  \end{equation*}
  giving the base of the induction.
  Now suppose that
  the lemma holds for dimensions $2, \ldots, k - 1$; we wish to show
  it holds for dimension $k$. Let $I = \{1, \ldots, k-1\}$ be the
  first $k-1$ indices, and for a vector $v \in \R^k$ let $v_I = (v_i)_{i \in I}$.
  Consider any collection $\{v\} \subset S$ covering the
  permutations~\eqref{eqn:all-permutation-orderings}; take two subsets
  $\mc{V}^1$ and $\mc{V}^2$ of these consisting (respectively) of those $v$
  such that $v_k \le v_j$ for all $j \le k-1$ and $v_k \ge v_j$ for all $j
  \le k-1$.
  Then by the induction, there exist $\wb{v}^i \in
  \conv(\mc{V}^i)$, $i = 1, 2$ such that
  \begin{equation*}
    \wb{v}^1 = \left[\begin{matrix} a \ones_{k-1} \\ s \end{matrix}
      \right] ~~ \mbox{and} ~~
    \wb{v}^2 = \left[\begin{matrix} b \ones_{k-1} \\ t \end{matrix}
      \right],
  \end{equation*}
  where $a(k - 1) + s = 0$ while $s \le a$ and $b(k - 1) + t = 0$
  while $t > b$.
  As $s \le 0$ and $t > 0$, so setting
  $\lambda = \frac{t}{t - s}$ gives
  $\lambda \wb{v}^1 + (1 - \lambda) \wb{v}^2 = \zeros$.
\end{proof}

In particular, we have shown that $\zeros$ minimizes the
surrogate risk~\eqref{eq:ranking-conditional-surrogate-risk},
and for any vector $q = (q_1, \ldots, q_k) > 0$ defining
$p = p(q)$ and $C$ in~\eqref{eqn:construction-p-ranking},
\begin{equation*}
  \inf_s R_\sloss(s \mid p) = R_\sloss(\zeros \mid p).
\end{equation*}
Notably, the minimizing vector $\zeros$ is independent of the parameters
$q_1, \ldots, q_k$ defining $p$ and $C$
in~\eqref{eqn:construction-p-ranking}.
For $(i, j) \in \mc{Y}$, let $D_{ij}
= \partial \varphi(\zeros, (i,j)) \subset \mathbb{R}^k$ be the set of
subgradients at $\zeros$, which is compact convex and non-empty.
Then by the
first-order optimality condition for subgradients and
construction~\eqref{eq:ranking-conditional-surrogate-risk} of the
conditional surrogate risk, there exist vectors $g_l \in D_{l,l+1}$
satisfying
\begin{equation*}
  \sum_{l=1}^{k-1} q_l g_l + \prn{1 - \sum_{l=1}^{k-1} q_l} g_k = 0
  ~~ \mbox{and~so}~~
  g_k = -\frac{\sum_{l=1}^{k-1} q_l g_l}{1 - \sum_{l=1}^{k-1} q_l} \in D_{k,1}.
\end{equation*}
As the $D_{ij}$ are compact convex, by taking
$q_k \uparrow 1$ and $(q_1, \ldots, q_{k-1}) \to \zeros$, we have
\begin{equation*}
  \ltwo{g_k} \leq
  \frac{\sum_{l=1}^{k-1} q_l}{1 - \sum_{l=1}^{k-1} q_l}
  \max_{i,j} \sup_{g \in D_{ij}} \ltwo{g} \to 0.
\end{equation*}
As the $D_{ij}$ are closed convex, we evidently have $\zeros \in D_{k,1}$,
while parallel calculation gives $\zeros \in D_{l,l+1}$ for each $l$.
A trivial modification to the construction~\eqref{eqn:construction-p-ranking}
to apply to cycles other than $(1, 2, \ldots, k, 1)$ then shows that
$\zeros$ minimizes $\varphi(\cdot, (i, j))$ for all pairs
$(i, j)$, violating Fisher consistency.

\subsection{Proof of Proposition~\ref{proposition:success-ranking}}
\label{sec:proof-success-ranking}

The proof relies on the fact when $m \geq k$, the event of observing
a comparisons $(i,j_i)$ for each $1 \leq i \leq k$ has nonzero probability.
Conditional on this event, we can obtain an
unbiased estimate of $C_x \ones$.
%
As $\sloss(s, \star) = 0$, it follows that
\begin{align*}
  R_{\sloss, A}(s\mid x) = \E \brk{\ltwo{s - A(Z)}^2
    \indic{A(Z) \neq \star}}
\end{align*}
When $m \geq k$, $\P(A(Z) \neq \star) > 0$, yielding per-$x$ minimizer
\begin{equation*}
  s\opt = \argmin R_{\sloss,A}(s \mid x)
  = \E\left[\left(\frac{m_{i1}}{m_1} + \cdots + \frac{m_{ik}}{m_k}
    \right)_{i \in [k]} \mid m_1 > 0, \ldots, m_k > 0 \right].
\end{equation*}
Conditioned on fixed positive values of $m_1, \cdots, m_k$,
\begin{equation*}
  (m_{1j}, \cdots, m_{qj}) \sim \multinom
  \left(m_j; \frac{p_{1j}}{\sum_{i=1}^k p_{ij}}, \ldots,
  \frac{p_{kj}}{\sum_{i=1}^k p_{lj}}\right),
\end{equation*}
so $\E[m_{ij}/m_j] = p_{ij}/\sum_{l=1}^k p_{lj} = (C_x)_{ij}$.
As $s\opt = C_x \ones$ is unique by convexity of the per-$x$ risk function, Fisher consistency follows.

%% file: consistency-proofs.tex
\section{Consistency proofs}

\subsection{Proof of Proposition~\ref{prop:comparison-inequality-agg}}
\label{sec:proof-comparison-inequality-agg}

Our only real assumption is that
$(s, x) \mapsto \tloss(s, P(\cdot \mid x))$ is jointly measurable
in $s$ and $x$.
Fix a function $f$.
Then for any $\epsilon > 0$,
\begin{align*}
  R_{\sloss,A}(f) - R\opt_{\sloss,A}
  = \int_{\mc{X}} \exslossagg(f(x), x) dP(x)
  \ge \int_{\extloss(f(x), x) \ge \epsilon} \psi_A(\epsilon, x) dP(x).
\end{align*}
Because $\psi_A(\epsilon, x) > 0$ for each $x$, the
measure defined by $d\nu(x) = b(x) dP(x)$ is
absolutely continuous with respect to
$d\mu(x) = \psi_A(\epsilon, x) dP(x)$.
That is, there exists $\delta > 0$ such that
$\nu(C) \le \epsilon$ for all $C \subset \mc{X}$
satisfying $\mu(C) \le \delta$.
Assume now that $R_{\sloss,A}(f) - R\opt_{\sloss,A} \le \delta$, so that
the set $\mc{X}_\epsilon \defeq \{x \mid \extloss(f(x), x) \ge \epsilon\}$,
which is measurable by the joint measurability assumption, satisfies
$\int_{\mc{X}_\epsilon} b(x) dP(x) \le \epsilon$.
We find
\begin{align*}
  R(f) - R\opt
  & = \int_{\extloss(f(x), x) \ge \epsilon}
  \extloss(f, x) dP(x)
  + \int_{\extloss(f(x), x) < \epsilon} \extloss(f(x), x) dP(x) \\
  & \le \int_{\mc{X}_\epsilon} b(x) dP(x)
  + \epsilon \le 2 \epsilon.
\end{align*}
In particular, we have shown that
$R_{\sloss,A}(f) - R_{\sloss,A}\opt \le \delta$
implies $R(f) - R\opt \le 2 \epsilon$,
which gives the ``hard'' direction of Fisher consistency.
The converse is trivial by considering a single $x$.

To see the comparison inequality,
note that by
definition of the calibration function,
\begin{align*}
  \psi_A(\extloss(f(x), x)) \leq \wb{\psi}_A(\extloss(f(x), x))
  \leq \exslossagg(f(x), x)
\end{align*}
for all $x \in \sX$.
The result follows by integrating on both sides w.r.t.\ $\dstr_X$ and
applying Jensen's inequality to $\psi_A$.

%% file: proof-comparison-consistency.tex


\subsection{The equivalence of the Mammen-Tsybakov conditions
  \eqref{eqn:mammen-alpha} and~\eqref{eqn:mammen-beta}}
\label{sec:mammen-alpha-beta}

We show the analogue of \citet[Thm.~3]{BartlettJoMc06},
essentially mimicking their proof, but including it for completeness.
\begin{lemma}
  Let $\alpha \in [0, 1]$.
  A distribution $P$ satisfies condition~\eqref{eqn:mammen-alpha}
  if and only if it satisfies condition~\eqref{eqn:mammen-beta}
  with $\beta = \frac{\alpha}{1- \alpha}$,
  where the constant $\cstMT$ may differ in the inequalities.
\end{lemma}
\begin{proof}
  Let condition~\eqref{eqn:mammen-alpha} hold.
  We let $c = \cstMT$ for shorthand and
  assume for notational
  simplicity that $y\opt(x) = \argmin_y \E[\tloss(y, Y) \mid X = x]$ is a
  singleton.
  Choose a measurable
  function $f$ such that
  \begin{align*}
    f(x) = y\opt(x), ~~\mbox{if} ~ \Delta(x) > \epsilon
    ~~ \mbox{and} ~~
    \extloss(f(x),x) = \Delta(x)
    ~~ \mbox{if}~ \Delta(x) \leq \epsilon.
  \end{align*}
  For all $\alpha \in [0,1]$, as $\extloss(f(x), x) = 0$
  if $\Delta(x) > \epsilon$,
  \begin{align*}
    \epsilon \P(\Delta(X) \leq \epsilon)
    \geq
    \E[\Delta(X) \indic{\Delta(X) \leq \epsilon}]
    & = \E\left[\extloss(f(X),X)\right]
    = R(f) - R\opt \\
    & \geq \prn{\frac{1}{c}
      \P(\pred \circ f \neq \pred \circ f\opt)}^{\frac{1}{\alpha}}
    = \prn{\frac{1}{c} \P(0 \leq \Delta(X) \leq \epsilon)}^{\frac{1}{\alpha}}.
  \end{align*}
  Rearranging terms we see for the constant $c' = c^{\frac{1}{\alpha}}$,
  \begin{align*}
    \mathbb{P}(0 \leq \Delta(X) \leq \epsilon) \leq
    (c' \epsilon)^{\alpha/(1-\alpha)},
  \end{align*}
  so condition~\eqref{eqn:mammen-beta} holds with $\beta = \frac{\alpha}{1 -
    \alpha}$.
  (The result is trivial when $\alpha = 1$, as
  $\P(\Delta(X) \le \epsilon) = 0$.)

  Now assume condition~\eqref{eqn:mammen-beta} holds for a value $0 \le \beta
  < \infty$, that is, $\P(\Delta(X) \leq \epsilon) \leq (c \epsilon)^\beta$
  for all $\epsilon > 0$.
  Recall the definition~\eqref{eqn:minimal-wrong-excess-risk}
  of $\Delta(x) = \min\{\extloss(s, x)
  \mid \pred(s) \not \in y\opt(x)\}$,
  so that
  \begin{align*}
    R(f) - R^\star
    & = \mathbb{E} \brk{\indic{\pred \circ f(X) \neq \pred \circ f^\star(X)}
      \extloss(f(X),X) } \\
    & \geq \mathbb{E} \brk{\indic{\pred \circ f(X) \neq \pred \circ f^\star(X)}
      \Delta(X) },
  \end{align*}
  and again by Markov's inequality for any $\epsilon \ge 0$,
  \begin{align}
    \E \brk{\indic{\pred \circ f(X) \neq \pred \circ f^\star(X)} \Delta(X)}
    & \geq  \epsilon \P(\pred \circ f(X) \neq \pred \circ f^\star(X),
    \Delta(X) > \epsilon) \nonumber \\
    & \geq \epsilon \prn{\mathbb{P}(\pred \circ f \neq \pred \circ f^\star)
      - \P(\Delta(X) \leq \epsilon)} \label{eqn:pinto-noir} \\
    & \geq \epsilon \P(\pred \circ f \neq \pred \circ f^\star)
    - c^\beta \epsilon^{1+\beta}, \nonumber
  \end{align}
  where the last inequality applies condition~\eqref{eqn:mammen-beta}.
  Maximizing the right hand side, we set
  \begin{align*}
    \epsilon = \frac{1}{c}\prn{\frac{\P(\pred \circ f \neq \pred \circ f^\star)}{
        (1+\beta)}}^{\frac{1}{\beta}}
  \end{align*}
  we obtain
  \begin{align*}
    R(f) - R^\star
    & \geq \frac{1}{c}
    \prn{\frac{\P(\pred \circ f \neq \pred \circ f^\star)}{
        (1+\beta)}}^{\frac{1}{\beta}}
    \cdot \prn{ \P(\pred \circ f \neq \pred \circ f^\star)
      - \frac{\P(\pred \circ f \neq \pred \circ f^\star)}{(1+\beta)}} \\
    & = \frac{\beta}{c(1+\beta)^{\frac{1+\beta}{\beta}}} \cdot
    \prn{\P(\pred \circ f \neq \pred \circ f^\star)}^{\frac{1 + \beta}{\beta}}.
  \end{align*}
  Set $c' = (c/\beta)^{\frac{\beta}{1+\beta}}(1+\beta)$,
  and recognize that $\log(1 + \beta) - \frac{\beta}{1 + \beta} \log \beta
  \le \log 2$ (so that $c'$ is indeed a constant),
  so that condition~\eqref{eqn:mammen-alpha} holds
  with $\alpha = \frac{\beta}{1 + \beta}$:
  \begin{align*}
    \P(\pred \circ f \neq \pred \circ f^\star)
    \leq c' (R(f) - R^\star)^{\frac{\beta}{1+\beta}}.
  \end{align*}
  When $\beta = \infty$,
  Condition~\eqref{eqn:mammen-beta}
  implies $P(\Delta(X) \leq 1/c) = 0$,
  so taking $\epsilon =1/c$ in inequality~\eqref{eqn:pinto-noir}
  \begin{align*}
    R(f) - R^\star &\geq \epsilon \prn{\mathbb{P}(\pred \circ f \neq \pred \circ f^\star) - \mathbb{P}(0 \leq \Delta(X) \leq \epsilon)} = \frac 1 c \mathbb{P}(\pred \circ f \neq \pred \circ f^\star).
  \end{align*}
  That is, condition~\eqref{eqn:mammen-alpha} with $\alpha = 1$ holds.
\end{proof}

\subsection{Proof of Theorem~\ref{thm:level-consistency-condition-number}}
\label{proof:level-consistency-condition-number}

The proof contains two parts.
In the first, we provide a lower bound for the calibration function
conditioning on $X=x$.
We then use the pointwise calibration function to prove a linear comparison
inequality on the data space $\sX^M := \{x \in \sX : \kappa(x) \leq M\}$ of
points with low noise condition number, and then conclude the proof via a
coarse risk bound on $\sX \backslash \sX^M$.

\paragraph{Part 1: lower bounding the pointwise calibration function.}
Before using properties of majority vote $A_m$, we start by assuming a
general aggregation method $A : \mc{Z} \to \{a_y\}_{y \in \mc{Y}}$.
To obtain the desired comparison inequality connecting
excess surrogate risk and task risk, we recall the pointwise calibration
function~\eqref{eqn:pointwise-agg-calibration},
\begin{align*}
  \wb{\psi}_A(\epsilon, x)
  \defeq \inf_{s \in \mathbb{R}^d}
  \left\{\exslossagg(s, x) : \extloss(s, x) \geq \epsilon\right\}.
\end{align*}
To lower bound $\wb{\psi}_A(\epsilon, x)$, we need to lower bound
$\delta_{\sloss, A}(s,x)$ provided that $\extloss(s,x) \geq
\epsilon$.
Because $\extloss (s,x) \geq \epsilon > 0$, it must hold that $\pred(s) \neq
y\opt$, which makes the following general lower bound, which applies for any
aggregation method and identifiable loss, useful:
\begin{lemma}
  \label{lemma:calib-lower-bound-1}
  Let $\sloss$ be $(\constant_{\sloss,1}, \constant_{\sloss,2})$-identifable
  (Def.~\ref{def:feasible}) with parameters
  $\{a_y\}_{y \in \mc{Y}}$
  and assume that $\pred(s) \neq y\opt$.
  Then for any aggregation method $A$,
  \begin{equation}
    \label{eq:calib-lower-bound-1}
    \delta_{\varphi, {A}} (s, x)
    \ge \constant_{\sloss,1} - (\constant_{\sloss,1} + \constant_{\sloss, 2})
    \P(A(Z) \neq a_{y\opt}).
  \end{equation}
\end{lemma}
\begin{proof}
  For the ground truth label $y\opt = y\opt(x)$,
  $\pred(s_{y\opt}) = y\opt$ by Definition~\ref{def:feasible},
  and
  \begin{align*}
    R_{\varphi,A}(s_{y^\star} \mid x) \geq \inf_{s} R_{\varphi,A}(s \mid x).
  \end{align*} 
  This allows us to bound the excess surrogate risk by
  \begin{align*}
    \delta_{\sloss, A} (s, x)
    & = R_{\varphi,A}(s\mid x) - \inf_{s} R_{\varphi,A}(s \mid x)
    \geq R_{\varphi,A}(s\mid x) - R_{\varphi,A}(s_{y\opt} \mid x) \\
    & = \Prb(A(Z) = a_{y\opt}) \prn{\varphi(s, a_{y^\star}) -
      \varphi(s_{y\opt}, a_{y\opt})} + \sum_{j \neq y\opt}  \Prb(A(Z) = a_j) \prn{\varphi(s, a_j) - \varphi(s_{y\opt}, a_j)}.
  \end{align*}
  Because by assumption $\pred(s) \neq y\opt$, the $(\constant_{\sloss, 1},
  \constant_{\sloss, 2})$-identifiability of $\sloss$ implies
  \begin{align*}
    \varphi(s, a_{y\opt}) - \varphi(s_{y\opt}, a_{y\opt})
    & \geq  \inf_{\pred \circ s \neq y\opt} \sloss(s, a_{y\opt})  - \varphi(s_{y\opt}, a_{y\opt}) \geq \constant_{\sloss, 1}, \\
    \varphi(s, a_j) - \varphi(s_{y\opt}, a_j) &
    \geq \inf_{s} \sloss(s, a_j) - \sloss(s_{y\opt}, a_j)
    \geq -\constant_{\sloss, 2},
  \end{align*}
  and therefore
  \begin{align*}
    \delta_{\varphi, {A}} (s, x)
    & \geq \Prb(A(Z) = a_{y\opt}) \constant_{\sloss, 1} - (1 - \Prb(A(Z) = a_{y\opt})) \constant_{\sloss, 2} \\
    & = \constant_{\sloss, 1} - (\constant_{\sloss, 1} + \constant_{\sloss, 2})  \Prb(A(Z) \neq a_{y\opt}),
  \end{align*}
  which is the lower bound~\eqref{eq:calib-lower-bound-1}.
\end{proof}

Equation~\eqref{eq:calib-lower-bound-1}
shows that to lower bound
$\exslossagg(s, x)$ when $\pred(s) \neq y\opt$,
it is sufficient to show that $A(Z) = a_{y\opt}$ with high probability.
For the majority vote~\eqref{eqn:naive-mv-aggregation},
as the number of labelers $m$ grow, the probability that
$\P(A_m(Z) = a_{y\opt}) \to 1$ by standard concentration
once we recall the definition~\eqref{eqn:minimal-wrong-excess-risk}
of the excess risk $\Delta(x) = \min_{\pred(s) \neq y\opt(x)}
\extloss(s, x)$.
\begin{lemma}
  \label{lemma:concentration-bound-majority-vote}
  Let $\card(\mc{Y}) = k$.
  For all $s \in \R^d, x \in \sX$ such that $\extloss (s,x) \geq
  \epsilon$,
  \begin{align*}
    \Prb(A_m(Z) \neq a_{y\opt}) \leq 2k \exp \prn{-m \Delta(x)^2/2}.
  \end{align*}
\end{lemma}
\begin{proof}
  Applying Hoeffding's inequality, simultaneously for each $y \in \mc{Y}$,
  \begin{align*}
    \left|\frac{1}{m} \sum_{l= 1}^m \tloss(y, Y_l) -
    \E[\tloss(y, Y) \mid X = x] \right| < \frac{\Delta(x)}{2}
  \end{align*}
  with probability at least $1 - 2k \exp \prn{-m\Delta(x)^2/2}$ as
  $\tloss \in [0,1]$.
  As $\extloss(s, x) = \E[\tloss(\pred(s), Y) - \tloss(y\opt, Y) \mid X = x]$,
  we have for all $y \neq y\opt$ that
  \begin{align*}
    \frac{1}{m} \sum_{l= 1}^m \tloss(y\opt, Y_l)
    < \frac{1}{m} \sum_{l= 1}^m \tloss(y, Y_l).
  \end{align*}
  Clearly the majority vote method $A_m(Z) = a_{y\opt}$ in this case.
\end{proof}

We can then substitute Lemma~\ref{lemma:concentration-bound-majority-vote}
into \eqref{eq:calib-lower-bound-1} and obtain a lower bound.
To also incorporate the condition $\extloss(s,x) \geq \epsilon$, we recall
the noise condition number~\eqref{eqn:noise-condition-number}, which
guarantees $\Delta(x) \geq \extloss(s,x) / \kappa(x)$ for all $s \in \R^d$.
This implies
\begin{align*}
  \delta_{\varphi, {A_m}} (s, x)  \geq \constant_{\sloss, 1} - 2k(\constant_{\sloss, 1}+\constant_{\sloss, 2}) e^{-\frac{m\Delta(x)^2}{2}}  \geq \constant_{\sloss, 1} - 2k(\constant_{\sloss, 1}+\constant_{\sloss, 2}) e^{-\frac{m \extloss(s,x)^2}{2\kappa(x)^2}},
\end{align*}
and thus
\begin{equation*}
  \wb{\psi}_A(\epsilon, x) \geq \constant_{\sloss, 1} - 2k(\constant_{\sloss, 1}+\constant_{\sloss, 2}) e^{-\frac{m\epsilon^2}{2\kappa(x)^2}}.
\end{equation*}

\paragraph{Part 2: restricting to $\sX^M$.}
Now it becomes clear why the error function $e_m(t)$ takes the form in
Eq.~\eqref{eq:error-function-mv}, as whenever
\begin{align*}
  \epsilon \geq e_m(\kappa(x)) = 	 \sqrt{\frac{2\kappa(x)^2 }{m}\log \prn{ \frac{4k(\constant_{\sloss, 1}+\constant_{\sloss, 2})}{\constant_{\sloss, 1}}}},
\end{align*}
we must have
\begin{align*}
  \wb{\psi}_A(\epsilon, x)  \geq \constant_{\sloss, 1} - 2k(\constant_{\sloss, 1}+\constant_{\sloss, 2}) e^{-\log \prn{ \frac{4k(\constant_{\sloss, 1}+\constant_{\sloss, 2})}{\constant_{\sloss, 1}}}} \geq \constant_{\sloss, 1} / 2,
\end{align*}
which further implies a pointwise convex lower bound $\wb{\psi}_A(\epsilon,
x) \geq \constant_{\sloss, 1} (\epsilon - e_m(\kappa(x)))_+/ 2$.
Restricting to $x \in \sX^M = \{x \in \mc{X} \mid \kappa(x) \le M\}$,
we clearly have
\begin{align*}
  \psi_{A_m}^M(\epsilon) \defeq
  \half \constant_{\sloss, 1}
  \hinge{\epsilon - e_m(M)}
  \le \wb{\psi}_A(\epsilon, x).
\end{align*}

Now, we proceed with an argument similar to those
\citet{BartlettJoMc06} use to provide fast rates of convergence
in binary classification
using $\psi_{A_m}^M(\epsilon)$ and applying over a restricted
data space $\sX^M$.

\begin{lemma}
  \label{lem:D-noise-comparison-inequality}
  Let $M > 0$ and for $f \in \mc{F}$,
  define $D(f, M) \defeq R(f) - R\opt - \P(\kappa(X) > M)$.
  Whenever $D(f, M) \ge 0$,
  \begin{equation*}
    \cstMT D(f, M)^\alpha \cdot
    \psi_{A_m}^M \prn{\frac{D(f,M)^{1-\alpha}}{2\cstMT}}
    \leq R_{\sloss,A_m}(f) - R_{\sloss,A_m}\opt.
  \end{equation*} 
\end{lemma}
\begin{proof}
  We begin with some generalities.
  By condition~\eqref{eqn:mammen-alpha}, for any function $f$ and
  $\epsilon > 0$,
  \begin{equation*}
    \E[\indic{\pred \circ f(X) \neq \pred \circ f\opt(X),
        \extloss(f(X), X) < \epsilon}
      \extloss(f(X), X)]
    \le \cstMT \epsilon \cdot (R(f) - R\opt)^\alpha,
  \end{equation*}
  so that
  \begin{align}
    R(f) - R\opt & = 
    \E\left[\indic{\pred \circ f(X) \neq y\opt(X)}
      \extloss(f(X), X)\right] \nonumber \\
    & \le \cstMT \epsilon \cdot (R(f) - R\opt)^\alpha
    + \E\left[\indic{\extloss(f(X), X) \ge \epsilon}
      \extloss(f(X), X)\right].
    \label{eqn:beethoven}
  \end{align}
  Consider the second term in the bound~\eqref{eqn:beethoven}.
  For any convex function $0 \le \psi$ with $\psi(0) = 0$,
  $\epsilon \mapsto \psi(\epsilon) / \epsilon$ is non-decreasing
  on $\epsilon > 0$ (cf.~\citet[Ch.~1]{HiriartUrrutyLe93}).
  This implies
  \begin{equation*}
    \frac{\psi(\epsilon)}{\epsilon}
    \indic{\extloss(f(x), x) \ge \epsilon}
    \le \frac{\psi(\extloss(f(x), x))}{\extloss(f(x), x)},
  \end{equation*}
  where we take $0/0 = 0$, and
  so $\psi(\epsilon) \extloss(f(x), x)
  \indic{\extloss(f(x), x) \ge \epsilon}
  \le \epsilon \cdot \psi(\extloss(f(x), x))$.
  Leveraging the calibration function~\eqref{eqn:pointwise-agg-calibration},
  if $\psi(\epsilon) \le \wb{\psi}_{\sloss, A_m}(\epsilon)$,
  then we evidently
  have
  \begin{align}
    \lefteqn{\psi(\epsilon)
      \E\left[\indic{\extloss(f(X), X) \ge \epsilon}
        \cdot \extloss(f(X), X)\right]} \nonumber \\
    & \le \epsilon \cdot
    \E\left[\psi(\extloss(f(X), X))\right]
    \le \epsilon \cdot \E\left[\exslossaggm(f(X), X)\right]
    = \epsilon \left(R_{\sloss,A_m}(f) - R_{\sloss,A_m}\opt\right).
    \label{eqn:bartok}
  \end{align}
  
  With these generalities in place, consider the function $f^M(x) = f(x)
  \indics{x \in \sX^M} + f\opt(x) \indics{x \not\in \sX^M}$.
  Substituting this in inequality~\eqref{eqn:beethoven} yields
  \begin{align}
    R(f^M) - R^\star
    & \leq \cstMT \epsilon \cdot (R(f^M)-R^\star)^\alpha +
    \mathbb{E}[\indic{\extloss(f^M(X),X) \geq \epsilon} \extloss(f^M(X),X)] .
    \label{eq:noise-risk-agg-3}
  \end{align}
  for all $\epsilon > 0$.
  Because the truncated calibration function $\psi_{A_m}^M$
  is convex,
  inequality~\eqref{eqn:bartok} yields
  \begin{align*}
    \psi_{A_m}^M(\epsilon)
    \E[\indic{\extloss(f^M(X),X) \geq \epsilon}
      \cdot \extloss(f^M(X),X)]
    & \leq \epsilon \cdot \E[\psi_{A_m}^M(\extloss(f^M(X),X))].
  \end{align*}
  Because $0 \le \psi_{A_m}^M \le \psi_{A_m}
  = \wb{\psi}_{A_m}^{**}$,
  we evidently obtain
  \begin{equation*}
    \E\left[\psi_{A_m}^M(\extloss(f^M(X), X))\right]
    \le R_{\sloss,A_m}(f) - R\opt_{\sloss,A_m}.
  \end{equation*}
  By inequality~\eqref{eq:noise-risk-agg-3},
  we therefore have
  \begin{equation*}
    \frac{R(f^M) - R\opt}{\epsilon}
    - \cstMT \left(R(f^M) - R\opt\right)^\alpha
    \le \frac{1}{\epsilon}
    \E\left[\indic{\extloss(f^M(X), X) \ge \epsilon}
      \extloss(f^M(X), X)\right],
  \end{equation*}
  and so multiplying by $\psi_{A_m}^M(\epsilon)$,
  \begin{align}
    \lefteqn{\psi_{A_m}^M(\epsilon)
      \left(\frac{R(f^M) - R^\star}{\epsilon} - \cstMT(R(f^M)-R^\star)^\alpha
      \right)} \nonumber \\
    & \le \frac{\psi_{A_m}^M(\epsilon)}{\epsilon}
    \E\left[\indic{\extloss(f^M(X),X) \geq \epsilon}
      \cdot \extloss(f^M(X),X)\right]
    \le R_{\sloss,A_m}(f^M) - R\opt_{\sloss, A_m}.
    \label{eqn:vpn-disconnected}
  \end{align}
  
  Finally, we use that $\epsilon$ was arbitrary.
  Taking $\epsilon = (R(f^M)-R^\star)^{1-\alpha}/(2\cstMT)$
  in inequality~\eqref{eqn:vpn-disconnected}
  gives
  $\psi_{A_m}^M(\epsilon)
  \cstMT (R(f^M) - R\opt)^\alpha \le
  R_{\sloss,A_m}(f^M) - R\opt_{\sloss,A_m}$.
  Using that
  \begin{equation*}
    D(f, M) \defeq R(f) - R\opt - \P(\kappa(X) > M)
    \le R(f^M) - R\opt
  \end{equation*}
  completes the proof of the lemma.
\end{proof}

We have nearly completed the proof of
Theorem~\ref{thm:level-consistency-condition-number}.
By the condition $R(f) - R\opt \geq 2\Prb(\kappa(X) > M) + (4\cstMT
e_m(M))^{\frac{1}{1-\alpha}}$ we have $D(f, M) \geq (R(f) - R\opt)/2$, while
at the same time
\begin{align*}
  \frac{D(f, M)^{1-\alpha}}{2 \cstMT} \geq 2 e_m(M).
\end{align*}
By convexity, $\psi_{A_m}^M(\epsilon)/\epsilon$ is non-decreasing in
$\epsilon$, so we further have
\begin{align*}
  \cstMT D(f, M)^\alpha \psi_{A_m}^M \prn{\frac{D(f,M)^{1-\alpha}}{2\cstMT}} & \geq \cstMT D(f, M)^\alpha \cdot \frac{D(f,M)^{1-\alpha}}{2\cstMT} \cdot \frac{\psi_{A_m}^M (2 e_m(M))}{2e_m(M)} \nonumber \\
  & = \frac{1}{2} D(f,M) \cdot \frac{1}{4} \constant_{\sloss, 1} \geq \frac{\constant_{\sloss, 1} (R(f) - R\opt)}{16}.
\end{align*}
Substitute the above display into
Lemma~\ref{lem:D-noise-comparison-inequality}.

\subsection{Proof of Corollary~\ref{corollary:level-consistency-optimize-M}}
\label{proof:level-consistency-optimize-M}

Recall that $k = \card(\mc{Y}) < \infty$.
For the binary case that $k=2$, we simply take $M=1$ and as
\begin{align*}
  4\cstMT e_m(1) = \prn{\frac{32 \cstMT^2}{m}\log \prn{ \frac{8(\constant_{\sloss, 1}+\constant_{\sloss, 2})}{\constant_{\sloss, 1}}}}^{\frac{1}{2(1-\alpha)}} =
  \xi_{m,2},
\end{align*}
Theorem~\ref{thm:level-consistency-condition-number} implies the conclusion.

For the general multiclass case, we can bound the tail probability by using
$\kappa(X) \leq 1/\Delta(X)$ and the low noise
condition~\eqref{eqn:mammen-alpha}, as $\Prb(\kappa(X) > M) \leq
\Prb(\Delta(X) \leq 1/M) \leq (\cstMT /
M)^{\frac{\alpha}{1-\alpha}}$.
Therefore, using Theorem~\ref{thm:level-consistency-condition-number},
we only need to prove
\begin{align*}
  \inf_M \{2 \cdot(\cstMT / M)^{\frac{\alpha}{1-\alpha}} +  (4\cstMT e_m(M))^{\frac{1}{1-\alpha}} \} \leq 	4 \cdot  \prn{\frac{32 \cstMT^4}{m}\log \prn{ \frac{4k(\constant_{\sloss, 1}+\constant_{\sloss, 2})}{\constant_{\sloss, 1}}}}^{\frac{\alpha}{2(1-\alpha^2)}}.
\end{align*}
Indeed, we choose the $M$ such that $2(\cstMT / M)^{\frac{\alpha}{1-\alpha}}
= (4\cstMT e_m(M))^{\frac{1}{1-\alpha}}$, which, by substituting in
Eq.~\eqref{eq:error-function-mv}, is equivalent to
\begin{align*}
	M^{-\frac{1+\alpha}{1-\alpha}} = \prn{\frac{32}{m}\log \prn{ \frac{4k(\constant_{\sloss, 1}+\constant_{\sloss, 2})}{\constant_{\sloss, 1}}}}^{\frac{1}{2(1-\alpha)}} \cdot \frac{\cstMT}{2},
\end{align*}
and thus we choose
\begin{align*}
  M =  \prn{\frac{32}{m}\log \prn{ \frac{4k(\constant_{\sloss, 1}+\constant_{\sloss, 2})}{\constant_{\sloss, 1}}}}^{-\frac{1}{2(1+\alpha)}} \cdot \prn{\frac{\cstMT}{2}}^{-\frac{1-\alpha}{1+\alpha}}.
\end{align*}
With this choice
\begin{align*}
  \lefteqn{2\cdot(\cstMT / M)^{\frac{\alpha}{1-\alpha}}+ (4\cstMT e_m(M))^{\frac{1}{1-\alpha}}} \\
  & = 2 \cdot (4\cstMT e_m(M))^{\frac{1}{1-\alpha}} \\
  & =  2 \cdot \prn{\frac{32 \cstMT^2}{m}\log \prn{ \frac{4k(\constant_{\sloss, 1}+\constant_{\sloss, 2})}{\constant_{\sloss, 1}}}}^{\frac{1}{2(1-\alpha)}} \cdot \prn{\frac{32}{m}\log \prn{ \frac{4k(\constant_{\sloss, 1}+\constant_{\sloss, 2})}{\constant_{\sloss, 1}}}}^{-\frac{1}{2(1+\alpha)(1-\alpha)}} \cdot \prn{\frac{\cstMT}{2}}^{-\frac{1}{1+\alpha}} \nonumber \\
  & =  2^{1 + \frac{1}{1+\alpha}} \cdot \prn{\frac{32}{m}\log \prn{ \frac{4k(\constant_{\sloss, 1}+\constant_{\sloss, 2})}{\constant_{\sloss, 1}}}}^{\frac{1}{2(1-\alpha)}} \cdot \prn{\frac{32}{m}\log \prn{ \frac{4k(\constant_{\sloss, 1}+\constant_{\sloss, 2})}{\constant_{\sloss, 1}}}}^{-\frac{1}{2(1+\alpha)(1-\alpha)}} \cdot \cstMT^{\frac{1}{1-\alpha}-\frac{1}{1+\alpha}} \nonumber \\
  & \leq 4 \cdot \prn{\frac{32}{m}\log \prn{ \frac{4k(\constant_{\sloss, 1}+\constant_{\sloss, 2})}{\constant_{\sloss, 1}}}}^{\frac{\alpha}{2(1-\alpha^2)}} \cdot \cstMT^{\frac{2\alpha}{1-\alpha^2}} \nonumber \\
  & = 4 \cdot  \prn{\frac{32 \cstMT^4}{m}\log \prn{ \frac{4k(\constant_{\sloss, 1}+\constant_{\sloss, 2})}{\constant_{\sloss, 1}}}}^{\frac{\alpha}{2(1-\alpha^2)}} = \xi_{m,k}.
\end{align*}

%% file: surrogate-examples-proofs.tex
\subsection{Proofs for the identifiable surrogate losses}
\label{sec:surrogate-example-proofs}

\subsubsection{Proof of Lemma~\ref{lemma:binary-feasible}}

That $\phi(\delta) < 0$ is immediate because $\phi$ is non-increasing by
assumption, and the monotonicity properties of convex
functions~\citep[Ch.~1]{HiriartUrrutyLe93} guarantee it strictly
decreases near 0.
For $y \in \{\pm 1\}$, we take
$s_y = \delta y$ and $a_y = y$,
and
\begin{align*}
  &\sloss(s_{y}, a_{y}) + \constant_{\sloss, 1}
  = \phi(\delta) +  \constant_{\sloss, 1} = \phi(0)
  = \inf_{\pred(s) \neq y} \sloss(s, a_y) \\
  &\sloss(s_y, a_{-y})
  - \constant_{\sloss, 2} \leq \phi(-\delta) -\constant_{\sloss, 2} = 0
  = \inf_{s \in \mathbb{R}} \sloss(s, a_y).
\end{align*}
by direct evaluation.

\subsubsection{Proof of Lemma~\ref{lemma:bipartite-feasible}}
  
For each $y \in \sY$, we choose $(s_y, a_y) =
(v(y)/N, y)$.
Observe that for each graph $y \in \mc{Y}$,
\begin{align*}
  \sloss(s_y, a_y)
  & = \max_{\hat y \in \sY}
  \prn{\frac{1}{2N} \ltwo{v(\hat y)- v(y)}^2 +
    \left\< v(\hat y) - v(y), v(y)/N\right\>} \\
  & = \max_{\hat y \in \sY} \prn{\frac{1}{2N} \ltwo{v(\hat y)}^2
    - \frac{1}{2N} \ltwo{v(y)}^2} \\
  & = 0,
\end{align*}
where we use $\ltwo{v(y)}^2 = N$ for all $y \in \sY$.
If $\pred(s) \neq y$, then there exists some $y' \neq y \in \sY$ such that
\begin{align*}
  \< v(y') - v(y), s \> \geq 0.
\end{align*}
This implies
\begin{equation*}
  \sloss(s, a_y)  = \sloss(s, y)
  \geq
  \prn{\frac{1}{2N} \lone{v(y')- v(y)} + \< v(y') - v(y), s\>}
  \geq \frac{1}{2N} \lone{v(y')- v(y)}  \geq \frac{1}{N}
\end{equation*}
because
distinct bipartite matchings differ on at least two edges.
This then implies condition~\eqref{eq:feasible-surrogate-1} holds with
$\constant_{\sloss, 1} = 1/N$.
The second condition~\eqref{eq:feasible-surrogate-2} holds for
$\constant_{\sloss, 2} = 2$ as
\begin{align*}
  \sloss(s_{y'}, a_y) = \sloss(v(y')/N, y) = \max_{\hat y \in \sY}
  \prn{\frac{1}{2N}\lone{v(\hat y)- v(y)} +
    \left\< v(\hat y) - v(y), v(y')/N\right\>}  \leq 2
\end{align*}
whenever $v^T \ones = N$.

\subsubsection{Proof of Lemma~\ref{lemma:structured-feasible}}

By definition of $\tau(y)$, for any $\epsilon > 0$ and each $y \in
\mc{Y}$, we can take $s_y \in \mathcal{S}(y)$ (by using homogeneity
and scaling) such that
\begin{equation}
  \max_{\hat y \neq y} \tloss(\hat y,y) /\< v(y) - v(\hat y), s_y \>
  = 1
  ~~ \mbox{and} ~~
  \min_{\hat y \neq y} \tloss(\hat y,y) /\< v(y) - v(\hat y), s_y \>
  > \frac{1}{\tau(y) + \epsilon}.
  \label{eqn:shostakovich-prelude}
\end{equation}
We take $a_y = y$. 

\paragraph{Controlling $\constant_{\sloss, 1}$.} Because
by assumption $\max_{\hat{y} \neq y} \tloss(\hat{y}, y) / \<v(y) -
v(\hat{y}), s\> = 1$,
\begin{align*}
  \sloss(s_y, y) & = \max_{\hat y \in \sY} \prn{\tloss(\hat y,y) +
    \left\< v(\hat y) - v(y), s_y\right\>}_+ \nonumber \\ & =
  \max_{\hat y \in \sY} \left\{\left\< v(y) - v(\hat y), s_y\right\>
  \cdot \prn{\tloss(\hat y,y) /\< v(y) - v(\hat y), s_y \> - 1}_+
  \right\} = 0.
\end{align*}
For any $s$ such that $\pred(s) \neq y$, there must exist $\hat y \neq y$
such that $\< v(y) - v(\hat y), s \> \geq 0$ and thus
\begin{align*}
  \sloss(s, y) \geq \tloss(\hat y, y) + \< v(y) - v(\hat y), s \>  \geq \min_{\hat y \neq y} \tloss(\hat y, y) = \min_{\hat y \neq y} \tloss(\hat y, y) + 	\sloss(s_y, y).
\end{align*}
Thus we can take $\constant_{\sloss, 1} =\min_{\hat y \neq y}
\tloss(\hat{y}, y)$.

\paragraph{Controlling $\constant_{\sloss, 2}$.} For any $y' \neq y$,
the $s_y$ satisfying inequality~\eqref{eqn:shostakovich-prelude}
yields
\begin{align*}
  \left\< v(\hat y) - v(y), {s_{y'}}\right\> & = \left\< v(\hat y) -
  v(y'), {s_{y'}}\right\> + \left\< v(y') - v(y), {s_{y'}}\right\>
  \leq \left\< v(y') - v(y), {s_{y'}}\right\> - \tloss(\hat{y}, y') \\
  & \le
  (\tau(y') + \epsilon) \cdot \tloss(y, y') - \tloss(\hat{y}, y'),
\end{align*}
By the normalization $0 \leq \ell \leq 1$, we have
\begin{align*}
  \sloss(s_{y'}, a_y) & = \max_{\hat y \in \sY} \prn{\tloss(\hat y,y) + \left\< v(\hat y) - v(y), {s_{y'}}\right\>}_+ \nonumber \\
  & \leq \max_{\hat y \in \sY} \prn{\tloss(\hat y,y) + (\tau(y') + \epsilon) \cdot \tloss(y, y')}_+ \leq \tau(y') + 1 + \epsilon\, .
\end{align*}
As $\epsilon > 0$ was arbitrary, we can take $\constant_{\sloss, 2} =
\max_{y \in \sY} \tau(y) + 1$.

\paragraph{The special case of $\ell_0$ task loss.}
Finally we are left to show if $v(y) \in \left\{0, 1\right\}^d$ and
$\ell(\hat y, y) = \frac{1}{2d} \norm{\hat y - y}_1$, we have $\tau(y) =
1$ for all $y$.
This is trivial in this case, as we can take $s_y = 2v(y)-1$, and for
$\hat{y} \neq y$,
\begin{align*}
  \< v(y) - v(\hat y), s_y\> = \< v(y) - v(\hat y), 2v(y) - 1 \>
  & = \< v(y) - v(\hat y), 2v(y)\> - \norm{v(y)}_2^2 +
  \norm{v(\hat y)}_2^2 \nonumber \\
  & = \norm{v(y)}_2^2 - 2 \< v(y),
  v(\hat y) \> + \norm{v(\hat y)}_2^2 \\
  & = \norm{v(y) -
    v(\hat y)}_2^2 = \norm{v(y) - v(\hat y)}_0 = 2d \ell(\hat y, y)
  ,
\end{align*}
where we again use the fact that for $0$-$1$ features, $\norm{v(y) -
  v(\hat y)}_2^2 = \norm{v(y) - v(\hat y)}_1$.
We see that
$\ell(\hat y, y)/ \< v(y) - v(\hat y), s_y\>$ is a constant and thus
$\tau(y) = 1$ for all $y \in \sY$.

%% file: knn-appendix.tex

\section{$K$-nearest-neighbors and general aggregation methods}
\label{sec:discussion-knn}

In this section, we adapt the results in Section~\ref{sec:consistency-agg}
to demonstrate a consistency result for $K$-nearest-neighbor methods using
an analogue of majority vote labeling.
We assume the surrogate $\sloss$ is $(\constant_{\sloss,1},
\constant_{\sloss,2})$-identifiable (Def.~\ref{def:feasible}) with
parameters $\{a_y\}_{y \in \mc{Y}}$ and that $k = \card(\mc{Y}) < \infty$.
Given a sample $(X_i, Y_i)_{i = 1}^n$ and a point $x \in \mc{X}$,
sort the indices so that $\dist(X_{(1)}, x) \le \dist(X_{(2)}, x)
\le \cdots \dist(X_{(n)}, x)$ (and label $Y_{(i)}$) similarly.
Then the nearest-neighbor aggregator of a point $x$ is
\begin{equation}
  \label{eqn:knn-aggregator}
  A_{n,K}(x) \defeq a_{\hat{y}},
  ~~
  \hat{y} = \argmin_{y \in \mc{Y}}
  \sum_{i = 1}^K \tloss(y, Y_{(i)}),
\end{equation}
and we define the surrogate risk
\begin{align*}
  R_{\sloss,n,K}(f)
  \defeq \E[\sloss(f(X), A_{n,K}(X))],
\end{align*}
where $A_{n,K}$ implicitly depends on an imagined sample $(X_i,
Y_i)_{i=1}^n$.
We warn the reader that, at some level, the surrogate consistency guarantee
we provide will implicitly essentially show that $K$-nearest-neighbors is
consistent so long as $K \to \infty$ while $K/n \to 0$, a familiar result
for multiclass classification and regression problems~\citep{Stone77,
  DevroyeGyLu96}.


We will demonstrate the following theorem.
\begin{theorem}
  \label{theorem:knn-works}
  Let the loss $\sloss$ be identifiable (Definition~\ref{def:feasible}),
  assume the excess risk~\eqref{eqn:minimal-wrong-excess-risk} satisfies
  $P(\Delta(X) > 0) = 1$.
  Let $K = K(n)$ and $n$ satisfy $K/n \to 0$ and $K \to \infty$ as $n \to
  \infty$.
  Then for all $\epsilon > 0$, there exists $N$ and $\delta > 0$ such that
  for all $n \ge N$,
  \begin{equation*}
    R_{\sloss,n,K}(f) - R\opt_{\sloss,n,K}
    \le \delta
    ~~ \mbox{implies} ~~
    R(f) - R\opt \le \epsilon
  \end{equation*}
  for all measurable $f$.
\end{theorem}

The theorem more or less follows from
the following comparison inequality.


\begin{lemma}
  \label{lemma:knn-risk-gap}
  Let $\sloss$ be $(\constant_{\sloss,1},
  \constant_{\sloss,2})$-identifiable, $\gamma > 0$ satisfy $\gamma \le
  \frac{\constant_{\sloss,1}}{2(\constant_{\sloss,1} +
    \constant_{\sloss,2})}$, and define the set
  \begin{equation*}
    \mc{X}_{n,K}^\gamma \defeq
    \{x \in \mc{X} \mid \P(A_{n,K}(x) \neq a_{y\opt(x)}) \le \gamma\}.
  \end{equation*}
  Then for all measurable $f$,
  \begin{equation*}
    R(f) - R\opt \le \frac{2}{\constant_{\sloss,1}}
    \left(R_{\sloss,n,K}(f) - R\opt_{\sloss,n,K}(f)\right)
    + \P(X \not \in \mc{X}_{n,K}^\gamma).
  \end{equation*}
\end{lemma}
\begin{proof}
  For $n, K \in \N$, define the pointwise risk gap
  \begin{equation*}
    \knnriskgap(s, x)
    \defeq \E\left[\sloss(s, A_{n,K}(x))\right]
    - \inf_{s'} \E\left[\sloss(s, A_{n,K}(x))\right],
  \end{equation*}
  where the expectation is over the nearest-neighbor
  aggregation~\eqref{eqn:knn-aggregator}, and for $\epsilon > 0$ define the
  pointwise calibration function
  \begin{equation*}
    \psi_{n,K}(\epsilon, x)
    \defeq \inf_{s \in \R^d}
    \left\{\knnriskgap(s, x) \mid \extloss(s, x) \ge \epsilon \right\}.
  \end{equation*}
  Because Lemma~\ref{lemma:calib-lower-bound-1} (in the proof of
  Theorem~\ref{thm:level-consistency-condition-number}) holds for any
  aggregation method, we see that
  \begin{equation*}
    \psi_{n,K}(\epsilon, x)
    \ge \constant_{\sloss,1}
    - (\constant_{\sloss,1} + \constant_{\sloss, 2}) \P(A_{n,K}(x) \neq a_{y\opt(x)})
  \end{equation*}
  for all $x \in \mc{X}$ and $\epsilon > 0$.
  Let $\mc{X}^\gamma = \mc{X}^\gamma_{n,K}$ for shorthand.
  Then in particular, because $\gamma > 0$ is small enough that
  $(\constant_{\sloss,1} + \constant_{\sloss,2}) \gamma < \constant_{\sloss,
    1} / 2$, we have $\psi_{n,K}(\epsilon, x) \ge \half \constant_{\sloss,1}$
  for $x \in \mc{X}^\gamma$.
  As a consequence, we can expand the risk
  \begin{align*}
    R(f) - R\opt
    & = \E[\extloss(f(X), X)] \\
    & \le \E\left[\extloss(f(X), X) \indic{\extloss(f(X), X) \ge \epsilon}\right]
    + \epsilon \\
    & \le \E\left[\extloss(f(X), X)
      \indic{\extloss(f(X), X) \ge \epsilon,
        X \in \mc{X}^\gamma}\right]
    + \P(X \not \in \mc{X}^\gamma)
    + \epsilon.
  \end{align*}
  As we assume $\tloss \le 1$, we see that
  $\psi_{n,K}(\extloss(f(x), x), x)
  \ge \half \constant_{\sloss,1} \cdot \extloss(f(x), x)$ when
  $\extloss(f(x), x) \ge \epsilon$
  and $x \in \mc{X}^\gamma$, giving the upper bound
  \begin{align*}
    R(f) - R\opt
    & \le \frac{2}{\constant_{\sloss,1}}
    \E\left[
      \psi_{n,K}(\extloss(f(X), X))
      \indic{X \in \mc{X}^\gamma}
      \right]
    + \P(X \not \in \mc{X}^\gamma)
    + \epsilon \\
    & \le \frac{2}{\constant_{\sloss,1}} \left(R_{\sloss,n,K}(f)
    - R\opt_{\sloss,n,K}(f)\right)
    + \P(X \not \in \mc{X}^\gamma) + \epsilon.
  \end{align*}
  As $\epsilon > 0$ was arbitrary we obtain the lemma.
\end{proof}

By Lemma~\ref{lemma:knn-risk-gap}, it is therefore sufficient
to show that for any fixed $\gamma > 0$,
$\P(X \not \in \mc{X}_{n,K}^\gamma) \to 0$.
But for this, we can simply rely on the results of \citet{Stone77}: by his
Theorems 1 and 2, because $\mc{Y}$ is finite,
$K$-nearest neighbors (when $K = K(n)$ satisfies $K \to \infty$ and
$K/n \to 0$) is consistent for estimating the conditional
distribution of $P(Y = y \mid X = x)$.
Because $\Delta(X) > 0$ with probability 1,
we see that $\P(A_{n,K}(x) \neq a_{y\opt(x)}) \to 0$ for all $x$ except
perhaps a null set,
and so \citeauthor{Stone77}'s results imply
$\P(X \not \in \mc{X}_{n,K}^\gamma) \to 0$.

%% file: proof-model-consistency.tex
\section{Proofs associated with model-based consistency}

\subsection{Proof of Proposition~\ref{proposition:failure-binary}}
\label{proof:failure-binary}

\newcommand{\pointmass}{\boldsymbol{\delta}}

We begin by considering the distribution $\dstr_{x_1, x_2}$, whose
$X$-marginal is supported only on two data points $\{x_1, x_2\} \subset
\mathcal{X}$.
The key idea is that by carefully choosing $x_1, x_2$ and the conditional
distribution of $Y \mid X=x$, the conditional surrogate losses
\begin{equation*}
  R_\sloss(t \mid x_i) \defeq
  \Ep [\sloss(f_{t \theta\opt}(X), Y) \mid X = x_i]
  = \Ep[\phi(Y \<t \theta\opt, X\>) \mid X= x_i],
  ~~~ i = 1, 2,
\end{equation*}
attain their minima at distinct $t$, and if their is a $\theta \not \in
\linspan\{\theta\opt\}$ for which $\E[\sloss(f_\theta(X), Y) \mid X = x_i] =
\inf_t \E[\phi(Y t) \mid X = x_i]$ for each $i$, then $f_\theta$ would
attain less surrogate risk than any point in $\linspan\{\theta\opt\}$.
To guarantee that $R(\theta) = P(Y\<X, \theta\> \le 0)$ has a unique up to
scaling---that only points in $\linspan\{\theta\opt\}$ minmize $R$---we
perturb
$P_{x_1, x_2}$ slightly by defining $P$ to have $X$-marginal
\begin{equation*}
  P(X \in \cdot) = \frac{1 - \delta}{2}
  \left[\pointmass_{x_1} + \pointmass_{x_2}\right]
  + \delta \normal(0, I_d),
\end{equation*}
where $\pointmass_x$ denotes a point mass at $x$ and $\delta \ge 0$ is a
value to be chosen.

\paragraph{The construction of $P_{x_1,x_2}$ and $P$.}
Without loss of generality, we take $\theta\opt = e_1$, the first
canonical basis vector.
For a value $\beta > 0$ to be defined, define the $Y$ conditional
probability
\begin{align*}
  \eta_\beta(x) = P(Y = 1 \mid X = x)
  \defeq \min \left\{\hinge{\half + \<x, e_1\>
  \prn{\beta|\<x, e_2\>| + 1}}, 1\right\}
\end{align*}
which projects $\half + \<x, e_1\>(\beta |\<x, e_2\>| + 1)$ onto $[0, 1]$
and satisfies $\eta_\beta(x) < \half$ if and only if $\<x, e_1\> < 0$ and
$\eta_\beta(x) > \half$ if and only if $\<x, e_1\> > 0$.
With this construction, $\theta\opt = e_1$ is evidently the unique unit
vector $u \in \sphere^{d-1}$ satisfying $\sgn(\<x, u\>) = \sgn(\eta_\beta(x)
- 1/2)$ for all $x$, so for any $\theta \not \in \linspan\{\theta\opt\}$,
\begin{align*}
  R(f_{\theta}) > R(f_{\theta\opt}).
\end{align*}

We can now provide the explicit construction of the distribution $P$.
Assume we may take the two points $x_1, x_2$ to satisfy
$\eta_\beta(x_1) = \frac{2}{3}$ and $\eta_\beta(x_2) = \frac{1}{3}$.
Then defining
\begin{align*}
  g_\phi(t) = \frac{2}{3}\phi(t) +\frac{1}{3} \phi(-t),
\end{align*}
which is a coercive convex function (and so has a compact interval of
minimizers), we write the surrogate risk of a vector $\theta$ for $P_{x_1,
  x_2}$ (recalling that $P(Y = y \mid x) = \eta_\beta(x)$) as
\begin{align*}
  \E_{P_{x_1,x_2}}[\sloss(f_\theta(X), Y)]
  & = \half g_\phi(\<\theta, x_1\>) + \half g_\phi(-\<\theta, x_2\>).
\end{align*}
By direct calculation, for any $\alpha > \half$ and $\beta > 0$,
the choices
\begin{equation}
  \label{eqn:particular-x-points}
  x_1 = \frac{1}{6} e_1
  ~~ \mbox{and} ~~
  x_2 = -\frac{1}{12 \alpha} e_1 + \frac{2 \alpha - 1}{\beta} e_2
\end{equation}
guarantee $\eta_\beta(x_1) = \frac{2}{3}$ and $\eta_\beta(x_2) = \frac{1}{3}$.

\paragraph{Minimizing surrogate risk along certain direction.}
We wish to show that the surrogate attains its minimum along
a direction $u$ nearly perpendicular to $\linspan\{\theta\opt\}$.
Let $u \in \sphere^{d-1}$ have coordinates
$u_j = \<u, e_j\>$.
We shall prove the following lemma:
\begin{lemma}
  \label{lemma:minimizers-small-u2}
  Let $x_1, x_2$ have definition~\eqref{eqn:particular-x-points}, $\beta >
  0$, and $P$ be defined as above.
  Then $\theta\opt = e_1$ yields $R(f_{\theta\opt}) = R\opt = \inf_f R(f)$,
  and there is a constant $C_\phi$ depending only on $\phi$ such that if
  $|u_2| \le C_\phi \beta |u_1|$, then
  \begin{equation*}
    \inf_t \E[\phi(Y t\<u, X\>)]
    > \inf_\theta \E[\phi(Y \<\theta, X\>)].
  \end{equation*}
\end{lemma}

We turn to the proof of the lemma.
Using the choices~\eqref{eqn:particular-x-points}
of $x_1$ and $x_2$ and defining $s_1 = \<u, x_1\> = \frac{1}{6} u_1$
and $s_2 = \<u, x_2\> = \frac{u_1}{12 \alpha}$,
it follows that for any $t \in \R$,
\begin{align*}
  \Ep_{\dstr_{x_1, x_2}}  \brk{\phi(Y \<tu , X\>)}
  & = \half\prn{\frac{2}{3} \phi\prn{t s_1} + \frac{1}{3} \phi\prn{-t s_1}}
  + \half \prn{ \frac{2}{3} \phi \prn{t s_2 } + \frac{1}{3} \phi\prn{-t s_2}}
  \\
  & = \half (g_\phi(ts_1) + g_\phi(ts_2)).
\end{align*}
For $w_1, w_2 \in \R$,
define the parameterized function
\begin{align*}
  h_\phi(w_1, w_2) = \half \inf_{t \in \real} \left\{g_\phi(tw_1)
  + g_\phi(tw_2)\right\},
\end{align*}
which corresponds to the minimal value of the risk $t \mapsto
\E_{P_{x_1,x_2}}[\phi(Y \<t \theta, X\>)]$ when $w_1 = \<\theta, x_1\>$ and
$w_2 = \<\theta, x_2\>$ for some vector $\theta \in \R^d$.
The convexity and coercivity of $g_\phi$ imply that $h_\phi(w_1, w_2)$ is
continuous on $\real^2 \backslash \left\{0\right\}$, it is homogeneous in
that $h_\phi(tw_1, tw_2) = h_\phi(w_1, w_2)$ for all $t \neq 0$, and
by construction,
\begin{align*}
  \inf_{t \in \real} \Ep_{\dstr_{x_1, x_2}}  \brk{\phi(Y \<tu , X\>)}
  = h_\phi(s_1, s_2).
\end{align*}
Moreover, it is immediate that
\begin{equation*}
  g_\phi\opt \defeq \inf_t g_\phi(t) = \inf_{w_1^2 + w_2^2 = 1}
  \inf_t \half \left\{g_\phi(t w_1) + g_\phi(t w_2)\right\}
  = \inf_{w_1^2 + w_2^2 = 1} h_\phi(w_1, w_2).
\end{equation*}

Let $\mc{G} = [a, b] = \argmin_t g_\phi(t)$, where
we must have $0 < a \le b < \infty$ as $\phi'(0) < 0$.
Then we set the value $\alpha \defeq \frac{b}{a} \ge 1$ in the
definition~\eqref{eqn:particular-x-points} of the points $x_1, x_2$.
Let $w_2 < \frac{1}{\alpha} w_1$; then if $w_1 \in \mc{G}$,
we must have $w_2 < \frac{b}{\alpha} = a$, and so
$w_2 \not \in \mc{G}$, and so at least one of
$w_1, w_2 \not \in \mc{G}$.
Enforcing the strict inequality $w_2 < \frac{1}{\alpha} w_1$,
we see that
\begin{align*}
  C_{\phi,\alpha}
  \defeq \inf_{\substack{w_1^2 + w_2^2 =1, \\ w_1 \geq 0, w_2 \leq \frac{3}{4\alpha} w_1}}
  h_\phi(w_1, w_2) >  \inf_{w_1^2 + w_2^2 = 1} h_\phi(w_1, w_2) = g_\phi\opt.
\end{align*}
Rewriting this in terms of the unit vector $u$ we have been considering,
whenever $|u_2| \leq \frac{\beta u_1}{24\alpha(2\alpha-1)}$,
\begin{align*}
  s_2 \leq
  \frac{u_1}{12\alpha} +
  \frac{2\alpha-1}{\beta} \cdot \frac{\beta u_1}{24\alpha(2\alpha-1)}
  \leq \frac{u_1}{8 \alpha} = \frac{3}{4\alpha} s_1,
\end{align*} 
and in this case
\begin{align*}
  \inf_{t \in \R} \Ep_{\dstr_{x_1, x_2}}  \brk{\phi(Y \<tu , X\>)}
  = h_\phi(s_1, s_2) \geq
  \inf_{\substack{w_1^2 + w_2^2 =1, \\ w_1 \geq 0, w_2 \leq \frac{3}{4\alpha} w_1}}
  h_\phi(w_1, w_2) = C_{\phi,\alpha} > g_\phi\opt.
\end{align*}

Now we restruct $u \in \sphere^{d-1}$ to the collection of vectors
satisfying $|u_2| \le \frac{\beta}{24 \alpha(2\alpha - 1)} |u_1|$, and show
that if $\theta \in \linspan\{u\}$, the surrogate risk cannot attain its
minimum.
Indeed, recalling the construction~\eqref{eqn:particular-x-points},
the matrix
\begin{align*}
  \begin{bmatrix}
    1/6 & -1/(12\alpha) \\
    0 & (2\alpha-1)/\beta
  \end{bmatrix}
\end{align*}
is invertible and we can find a $\wb{\theta}$ such that $\<\wb{\theta},
x_1\> = c, \<\wb{\theta}, x_2\> = -c$ for some value
$c \in \mc{G} = \argmin g_\phi$, implying
\begin{align*}
  \inf_{t \in \R} \Ep_{\dstr} \brk{\phi(Y \<t u, X\>)}
  & \geq (1-\delta) \inf_{t \in \R} \Ep_{\dstr_{x_1, x_2}} \brk{\phi(Y \<t u, X\>)}
  =(1- \delta) C_{\phi,\alpha} \\
  \inf_{\theta \in \real^d} \Ep_{\dstr} \brk{\phi(Y \<\theta, X\>)}
  & \leq \Ep_{\dstr} \brk{\phi(Y \<\wb{\theta}, X\>)}
  = (1-\delta) g_\phi\opt +
  \delta  \Ep_{\normal(0, I_d)} \brk{\phi(Y \<\wb{\theta}, X\>)}.
\end{align*}
By taking $\delta$ sufficiently small and using
$C_{\phi,\alpha} > g_\phi\opt$, we conclude that
Lemma~\ref{lemma:minimizers-small-u2} holds with
$C_\phi = \frac{1}{24 \alpha(2 \alpha - 1)}$
and recognizing that $\alpha > \half$ was otherwise arbitrary.

\paragraph{Controlling the angle between $\theta_\sloss$ and $\theta\opt$.}
By Lemma~\ref{lemma:minimizers-small-u2},
there exists a constant $C_\phi$ such that for any $\beta > 0$,
we can construct a distribution $P$ for which
any minimizer $\theta_\sloss$ of the surrogate risk must satisfy
\begin{align*}
  |\< \theta_\sloss, e_2 \> |
  \geq C_\phi \beta \cdot | \< \theta_\sloss, e_1 \> |
  ~~ \mbox{and} ~~ |\< \theta_\sloss, e_2 \> | > 0 \, .
\end{align*}
Now we specify the parameter $\beta$, taking $\beta =
\frac{1}{C_\phi} \sqrt{\frac{1}{\epsilon^2} - 1}$.
Then evidently
\begin{align*}
  \<\theta_\sloss, e_2 \>^2 \geq
  \prn{\frac{1}{\epsilon^2} - 1} \cdot \<\theta_\sloss, e_1 \>^2,
\end{align*}
which combined with $\theta\opt = e_1$ implies
\begin{equation*}
  \left|\cos \angle (\theta_\sloss, \theta\opt) \right|
  = \frac{| \< \theta_\sloss, e_1 \> |}{ \|\theta_\sloss\|_2}
  \leq \frac{|\<\theta_\sloss, e_1 \> |}{
    \sqrt{| \<\theta_\sloss, e_1 \> |^2 + | \<\theta_\sloss, e_2 \> |^2}}
  \leq \frac{1}{\sqrt{1 + \frac{1}{\epsilon^2} - 1}} = \epsilon.
\end{equation*}
Because $\theta_\sloss \not \in \linspan\{\theta\opt\}$, we see that
$R(f_{\theta_\sloss}) > R(f_{\theta\opt}) = \inf_f R(f)$, completing the
proof of Proposition~\ref{proposition:failure-binary}.

\subsection{Proof of Theorem~\ref{thm:restricted-class}}
\label{proof:restricted-class}

\newcommand{\gapfunc}{\wb{\psi}_\sloss}  
\newcommand{\convgapfunc}{\psi_\sloss}  
\newcommand{\simplex}{\mc{P}}  

Let $\simplex_k = \{p \in \R^k_+ \mid \<\ones, p\> = 1\}$ by the probability
simplex in $\R^k$.
For $p \in \simplex_k$, define the risk gaps
\begin{equation*}
  \exsloss(s, p) \defeq
  \E_p[\sloss(s, Y)] - \inf_s \E_p[\sloss(s, Y)]
  ~~ \mbox{and} ~~
  \extloss(s, p) \defeq
  \E_p[\tloss(\pred(s), Y)]
  - \inf_s \E_p[\tloss(\pred(s), Y)]
\end{equation*}
and the gap functional
\begin{equation*}
  \gapfunc(\epsilon, p)
  \defeq \inf_s \left\{\exsloss(s, p) \mid \extloss(s, p) \ge \epsilon\right\}.
\end{equation*}
By the assumption that $\sloss$ is consistent, it is
immediate~\citep{Steinwart07} that $\gapfunc(\epsilon, p) > 0$ for all $p \in
\simplex_k$ and $\epsilon > 0$.
Moreover, consistency implies~\citep{Zhang04a} that if
$p_{(1)} \ge p_{(2)} \ge \cdots \ge p_{(k)}$ denotes the order
statistics of $p \in \simplex_k$, when we define the subset
\begin{equation*}
  \simplex_{k,c}
  \defeq \{p \in \simplex_k
  \mid p_{(1)} \ge p_{(2)} + c \}
\end{equation*}
of well-separated distributions, then for all $c > 0$ we have
the strict inequality
\begin{equation*}
  \inf_{p \in \simplex_{k,c}} \gapfunc(\epsilon, p) > 0
  ~~ \mbox{when}~ \epsilon > 0.
\end{equation*}

For $m \in \N$, let $P_m(\cdot \mid x)$ denote the
induced distribution on the majority vote $Y_m^+ \defeq
\mbox{Majority}(Y_1^m)$ for $Y_i \simiid P(Y \in \cdot \mid X = x)$,
so that if $\Delta(x) > 0$ we see that
$Y_m^+ \to y\opt(x)$ with probability 1.
Then
\begin{align*}
  \extloss(s, P_m(\cdot \mid x))
  & = \E_{P_m}[\indic{\pred(s) \neq Y_m^+} \mid x]
  - (1 - P(Y_m^+ = y\opt(x) \mid x)) \\
  & = \begin{cases}
    0 & \mbox{if}~ \pred(s) = y\opt(x) \\
    P(Y_m^+ = y\opt(x) \mid x) - P(Y_m^+ = \pred(s) \mid x)
    & \mbox{otherwise}.
  \end{cases}
\end{align*}
In particular, for $P$-almost-all $x$, we see that
$\extloss(s, P_m(\cdot \mid x)) \to \indic{\pred(s) = y\opt(x)}$
as $m \to \infty$.
Now, fix $c > 0$ and define
\begin{equation*}
  \gapfunc(\epsilon) \defeq
  \inf_{p \in \simplex_{k,c}} \gapfunc(\epsilon, p)
  ~~ \mbox{and} ~~
  \convgapfunc(\epsilon) \defeq
  \gapfunc^{**}(\epsilon),
\end{equation*}
the convex conjugate of the gap functional on well-separated
distributions.
Then~\citet[Prop.~25]{Zhang04a} shows that $\gapfunc^{**}(\epsilon) > 0$
whenever $\gapfunc(\epsilon) > 0$.

We now consider the gaps in the surrogate risk $R_{\sloss,A_m}(f) -
R_{\sloss,A_m}\opt$.
Letting $c > 0$, define
\begin{equation*}
  \mc{X}_{c,m} \defeq \left\{x \mid
  P_m(Y^+_m \in \cdot \mid X = x) \in \simplex_{k,c}\right\}
\end{equation*}
to be those $x \in \mc{X}$ for which the majority vote is likely correct.
Then
\begin{align*}
  \lefteqn{R_{\sloss, A_m}(f) - R_{\sloss, A_m}\opt} \\
  & = \E[\exsloss(f(X), P_m(\cdot \mid X))] \\
  & \ge \E[\gapfunc(\extloss(f(X), P_m(\cdot \mid X)), P_m(\cdot \mid X))] \\
  & \ge \E\left[\indic{X \in \mc{X}_{c,m}}
    \convgapfunc(\extloss(f(X), P_m(\cdot \mid X)))
    + \indic{X \not\in \mc{X}_{c,m}}
    \gapfunc(\extloss(f(X), P_m(\cdot \mid X)))\right] \\
  & \ge \E\left[\indic{X \in \mc{X}_{c,m}}
    \convgapfunc(\extloss(f(X), P_m(\cdot \mid X)))\right].
\end{align*}
Using Jensen's inequality that for any convex $h$, random variable
$Z$, and set $A$,
$\E[\indic{Z \in A} h(Z)]
= \E[h(Z) \mid Z \in A] P(Z \in A)
\ge h(\E[Z \mid Z \in A]) P(Z \in A)$,
we therefore obtain that
\begin{align}
  \nonumber
  R_{\sloss, A_m}(f) - R_{\sloss, A_m}\opt
  & \ge \convgapfunc\left(
  \E\left[\extloss(f(X), P_m(\cdot \mid X)) \mid X \in \mc{X}_{c,m}\right]
  \right) P(X \in \mc{X}_{c,m}) \\
  & = \convgapfunc\left(
  \frac{R(f) - R\opt - \E[\extloss(f(X), P_m(\cdot \mid X))
      \indic{X \not\in \mc{X}_{c,m}}]}{
    P(X \in \mc{X}_{c,m})}\right)
  P(X \in \mc{X}_{c,m}) \nonumber \\
  & \ge
  \convgapfunc\left(
  \frac{\hinge{R(f) - R\opt - P(X \not\in \mc{X}_{c,m})}}{
    P(X \in \mc{X}_{c,m})}\right)
  P(X \in \mc{X}_{c,m}).
  \label{eqn:majority-risk-gap}
\end{align}

Let $R_{\sloss,\infty}(f) = \E[\sloss(f(X), y\opt(X))]
= \lim_{m \to \infty} \E[\sloss(f(X), Y_m^+)]$.
Then if $f\opt \in \mc{F}$ is any
function with $\argmax_y f_y\opt(x) = y\opt(x)$ (for $P$-almost all $x$),
we evidently obtain
\begin{equation*}
  \lim_{t \to \infty} R_{\sloss,\infty}(t f\opt) = 0
\end{equation*}
by dominated convergence, as by
assumption~\eqref{eqn:multiclass-surrogate-limit} we have $\sloss(t
f\opt(x), y\opt(x)) \to 0$ as $t \to \infty$ for almost all $x$.
Let $\epsilon > 0$ be arbitrary
and take any $t < \infty$ large enough that $R_{\sloss,\infty}(t f\opt) \le
\epsilon$.
Then because $R_{\sloss, A_m}(t f\opt) \to R_{\sloss,\infty}$ as $m \to \infty$,
for the sequence $f_m \in \epsargmin[m] R_{\sloss,A_m}$,
we obtain
\begin{align*}
  R_{\sloss, A_m}(f_m) \le R_{\sloss,m}(t f\opt) + \epsilon_m
  \to R_{\sloss,\infty}(t f\opt) \le \epsilon.
\end{align*}
Substituting into inequality~\eqref{eqn:majority-risk-gap},
we have
\begin{align*}
  \epsilon \ge \limsup_m R_{\sloss, A_m} - R_{\sloss,A_m}\opt
  & \ge
  \limsup_m \convgapfunc\left(
  \frac{\hinge{R(f_m) - R\opt - P(X \not\in \mc{X}_{c,m})}}{P(X \not \in
    \mc{X}_{c,m})}\right)
  P(X \in \mc{X}_{c,m}).
\end{align*}
Because $P(X \not \in \mc{X}_{c,m}) \to 0$ by assumption,
if $\limsup_m R(f_m) - R\opt = \delta > 0$, we would obtain
\begin{equation*}
  \epsilon \ge \convgapfunc(\delta).
\end{equation*}
But $\convgapfunc(\delta) > 0$ for $\delta > 0$, and $\epsilon > 0$ was
arbitrary, so it must be that
$\limsup_m R(f_m) = R\opt$.

%% file: proof-misspecified-finite-dim.tex
\section{Proofs for mis-specified models}

\subsection{Proof of
  Proposition~\ref{proposition:consistency-failure-multiindex}}
\label{sec:proof-consistency-failure-multiindex}

\newcommand{\sigmauni}{\sigma^{\textup{uni}}}

We assume the result of Theorem~\ref{thm:maj-multiindex},
as its proof does not depend on the current proposition.
To simplify the proof and work with square matrices, we assume
w.l.o.g.\ that $\Theta\opt = U\opt T\opt$, where $U\opt \in \R^{d \times
  (k-1)}$ is orthogonal, and we may w.l.o.g.\ take $T\opt$ to be diagonal,
with $T\opt = \diag(t_1\opt, \ldots, t_{k-1}\opt)$, and let
$\Theta_1(\epsilon) = U^\star T_1(\epsilon)$.
It suffices to show that $T_1(\epsilon) / \norm{T_1(\epsilon)} \neq T^\star /
\norm{T^\star}$.
For simplicity, we suppress the dependence on $m = 1$ and
write $T(\epsilon) = T_1(\epsilon)$, $\Theta(\epsilon) =
\Theta_1(\epsilon)$, and let
$T_{ij}(\epsilon)$ denote the entries of $T(\epsilon)$.
As $X \sim \normal(0, I_d)$, it follows that ${U\opt}^\top X \sim \normal(0,
I_{k-1})$, so that the stationary conditions for $\Theta(\epsilon)$
equivalently state that for $Z \sim \normal(0, I_{k-1})$,
\begin{align*}
  \nabla_\Theta L_{1,\epsilon}(\Theta(\epsilon))
  = \E \brk{Z \prn{ \sigma\lr(T(\epsilon)^\top Z)
      - \sigma\lre({T\opt}^\top Z)}^\top} = \zeros_{d \times (k-1)}.
\end{align*}
Let $\mc{T} \subset \R^{k-1}$ be a set
to be chosen, and write
$\sigma (t) = \sigma\lr (t) \indic{t \not \in \mc{T}}
+ \sigmauni \cdot \indic{t \in \mc{T}}$,
where $\sigmauni = \frac{1}{k} \ones$ denotes the uniform distribution.
Then equivalently,
\begin{align*}
  \Ep \brk{Z \prn{ \sigma\lr(T(\epsilon)^\top Z)
      - \sigma\lr({T\opt}^\top Z)}^\top}
  + \underbrace{\Ep \brk{Z \prn{ \sigma\lr({T\opt}^\top Z)
        - \sigmauni}^\top \indic{{T\opt}^\top Z \in \mc{T}}}}_{
    =: A(\mc{T})} = \zeros.
\end{align*}
For small $\epsilon > 0$, we can always choose disjoint $\mc{T}_\epsilon$
and $\mc{T}_{-\epsilon}$ with $P({T\opt}^\top Z \in \mc{T}_\epsilon),
P({T\opt}^\top Z \in \mc{T}_{-\epsilon}) \le \epsilon$ while the matrices
$A(\mc{T}_\epsilon)$ and $A(\mc{T}_{-\epsilon})$ belong to distinct rays,
that is, are not positive multiples of one another.
Indeed, as the rank one matrix ${T\opt}^\top Z(\sigma\lr({T\opt}^\top Z) -
\sigmauni)^\top$ is non-constant whenever $k \geq 3$, we can find a matrix
$Q \in \R^{(k-1) \times (k-1)}$ such that the sets
\begin{align*}
  \mc{T}_+ & = \left\{{T\opt}^\top z \mid \tr\left(Q^\top
  {T\opt}^\top z (\sigma\lr({T\opt}^\top z) - \sigmauni)^\top\right)
  > 0 \right\} \\
  \mc{T}_- & = \left\{{T\opt}^\top z \mid \tr \left(Q^\top {T\opt}^\top z
  (\sigma\lr({T\opt}^\top z) - \sigmauni)^\top\right) < 0 \right\}
\end{align*}
have positive Lebesgue measure.
Then for any $\mc{T}_\epsilon \subset \mc{T}_+$ and $\mc{T}_{-\epsilon}
\subset \mc{T}_-$, we must have $\tr(Q^\top A(\mc{T}_\epsilon)) > 0$ and
$\tr(Q^\top A(\mc{T}_{-\epsilon})) < 0$, as desired, and we may take the
sets $\mc{T}_{\pm \epsilon}$ to have Lebesgue measure at most $\epsilon$.
By absolute continuity of Lebesgue integral, as $\epsilon \to 0$ it follows
$A(\mc{T}_\epsilon) \to 0$ and $A(\mc{T}_{-\epsilon}) \to 0$ uniformly with
$\epsilon$.

Now we are ready to prove the lemma.
Consider the tilted gradient function
\begin{align*}
  F(T,A) = \Ep \brk{Z \prn{ \sigma\lr(T^\top Z)
      - \sigma\lr({T\opt}^\top Z)}^\top} + A,
\end{align*}
which satisfies $F(T\opt, 0)= 0$, and for which the linear mapping
\begin{equation*}
  D(T) =
  \nabla_T F(T, 0) : \R^{(k-1) \times (k-1)} \to \R^{(k - 1) \times (k - 1)},
  ~~~
  D(T)[M] \defeq \E[Z (\nabla \sigma\lr(T^\top Z) M Z)^\top]
\end{equation*}
is invertible at $D(T\opt)$.
By construction of the matrices
$A_{\pm \epsilon} \defeq A(\mc{T}_{\pm \epsilon})$,
we also know that
there exist solutions $T(\pm \epsilon)$ satisfying
$F(T(\epsilon), A_\epsilon) = 0$ and
$F(T(-\epsilon), A_{-\epsilon}) = 0$.
By the implicit function theorem and that $\nabla_A F(T, A) = \textup{Id}$,
we thus obtain
\begin{align*}
  T(\epsilon) & = T\opt - D(T\opt)^{-1} \nabla_A F(T\opt, 0)
  A_\epsilon + O(\norm{A_\epsilon}^2) \\
  & = T\opt - D(T\opt)^{-1} A_\epsilon + O(\norm{A_\epsilon}^2),
\end{align*}
and similarly $T(-\epsilon) = T\opt - D(T\opt)^{-1} A_{-\epsilon} +
O(\norm{A_{-\epsilon}}^2)$.
Without explicitly computing the Jacobian, we may still conclude that at
least one of $T(\epsilon)$ and $T(-\epsilon)$ cannot align with $T\opt$, as
$T(\epsilon) - T^\star$ and $\wb{T}(\epsilon) - T^\star$ belong to distinct
rays.

\subsection{Proof of Theorem~\ref{thm:maj-multiindex}}
\label{proof:maj-multiindex}

We prove the theorem in two parts. In the first we verify the validity of
the ansatz $\Theta_m = U\opt T_m$, and in the second we show the claimed
asymptotics of $T_m$.

\paragraph{Part 1: Ansatz for the population loss.}
Let $Z = {U\opt}^\top X \sim \normal(0, {U\opt}^\top \Sigma U\opt)$, and let
$A \in \R^{d \times (k-1)}$ satisfy
\begin{align*}
  0 = \cov(X - AZ, Z) = \Sigma U^\star  - A {U^\star}^\top \Sigma U^\star,
\end{align*}
i.e., $A = \Sigma U^\star ({U^\star}^\top \Sigma U^\star)^{-1}$.
Then $X-AZ$ and $Z$ are independent. Consider the lower dimensional problem
in $\R^{k-1}$ with the covariates $X$ replaced by $Z$ and $\Theta^\star$
replaced by $T^\star$,
with associated loss (abusing notation)
\begin{equation*}
  \wb{L}_m(T) \defeq \E\left[-\log P_T(Y_m^+ \mid Z)\right]
  = \E\left[\sloss(T^\top Z, Y_m^+)\right],
\end{equation*}
where $Y_m^+$ denotes majority vote and $P_T$ the logistic regression model.
The loss $L_m$ is still strictly  convex and coercive, so
it has unique minimum $T_m \in \R^{(k-1) \times (k-1)}$ satisfying
\begin{align*}
  \nabla_{\Theta} \wb{L}_m(T_m)
  = \E\left[Z (\sigma\lr(T_m^\top Z) - \rho_m({T\opt}^\top Z))^\top\right]
  = 0,
\end{align*}
where we recall the notation that
$\rho_m(v) = (P(Y_m^+ = 1), \ldots, P(Y_m^+ = k))$
when $Y_i \simiid \cat(v)$.

We demonstrate $\Theta_m = U\opt T_m$ minimizes $L_m$.
Indeed,
\begin{align*}
  \nabla_{\Theta} L_m(\Theta_m)
  & = \Ep \brk{X \prn{ \sigma\lr(T_m^\top Z) - \rho_m({T\opt}^\top Z)}^\top} \\
  & = A \underbrace{\Ep \brk{Z \prn{ \sigma\lr(T_m^\top Z)
        - \rho_m({T\opt}^\top Z)}^\top}}_{
    = 0~\text{by stationarity of }\wb{L}_m}
  + \Ep \brk{(X-AZ) \prn{ \sigma\lr(T_m^\top Z) - \rho_m({T\opt}^\top Z)}^\top} \\
  & \stackrel{(\star)}{=}
  \underbrace{\Ep [X-AZ]}_{=0} \cdot \,
  \Ep[\sigma\lr(T_m^\top Z) - \rho_m({T\opt}^\top Z)]^\top = 0,
\end{align*}
where equality~$(\star)$ uses the independence between $X-AZ$ and $Z$.

\paragraph{Part 2: Asymptotics of $T_m$.}
We prove $\|T_m\| \to \infty$ and $T_m/\|T_m\| - T^\star / \|T^\star\| \to
0$.
\begin{lemma}
  \label{lemma:norms-biggeroo}
  Under the conditions of the theorem,
  $\Theta_m = \argmin_\Theta L_m(\Theta)$
  satisfies
  $\opnorm{\Theta_m} = \opnorm{T_m} \to \infty$
  and $L_m(\Theta_m) \to 0$.
\end{lemma}
\begin{proof}
  When $\opnorm{T} \le r$, for $\Theta = U^\star T$ and $i, j \in [k]$ we have
  \begin{align*}
    |\<\theta_i - \theta_j, x\>|
    = |\<U\opt T (e_i - e_j), x\>|
    \le \ltwo{e_i - e_j} \opnorm{T} \ltwo{x}
    \le \sqrt{2} r \ltwo{x}
  \end{align*}
  Therefore we have pointwise lower bound for the loss
  \begin{align*}
    \sloss(\Theta^\top x, y)
    & = \log \bigg(\sum_{i=1}^{k} \exp(\<\theta_i - \theta_y, x\>)\bigg)
    \geq \log \prn{1 + (k-1) \exp\{-\sqrt{2} r \ltwo{x}} > 0.
  \end{align*}
  Letting $g(r) := \E[\log(1 + (k-1) \exp(-\sqrt{2} r \ltwo{X}))] > 0$,
  which is a strictly decreasing function of $r$,
  we see that for all $m \in \N$ and $\opnorm{\Theta}
  = \opnorm{T} \le r$, $L_m(\Theta) \geq g(r)$.

  On the other hand, for a real number $R > 0$, consider $\Theta_R \defeq R
  \Theta\opt / \opnorm{\Theta\opt}$, whose columns $\theta_1, \ldots,
  \theta_k$ are scaled multiples of those of $\Theta\opt$.
  It is clear from majority vote consistency that $\rho_m(\Theta_R x) \to
  e_{y\opt(x)}$ as $m$ or $R \to \infty$, and so
  \begin{align*}
    L_m(\Theta_R) & = \Ep [\sloss(\Theta_R^\top X, Y_m^+)] \\
    & \mathop{\longrightarrow}_{m\uparrow \infty}
    \Ep \brk{\log \prn{\sum_{i=1}^{k} \exp(\<\theta_i - \theta_{y^\star(X)}, X\>)}}  \nonumber \\
    & \leq \Ep \brk{\log \prn{1 + (k-1) \exp\big(
        -R \min_{i \neq j} \ltwo{\theta\opt_i - \theta\opt_j}
        \ltwo{X} / \opnorm{T\opt} \big)}}
    =: h(R).
  \end{align*}
  We conclude that
  \begin{equation*}
    \limsup_m \inf_\Theta L_m(\Theta) \le
    h(R)
  \end{equation*}
  This implies for sufficiently large $m$, $\inf_\Theta
  L_m(\Theta) < 2h(R)$ and we must have $\norm{\Theta_m} \geq g^{-1}(2h(R))$.
  As both $g$ and $h$ monotonically decrease to $0$ on $\R_+$, we see that
  $\norm{\Theta_m} \to \infty$.
  The unitary invariance of $\opnorm{\cdot}$ gives that
  $\opnorm{\Theta_m} = \opnorm{T_m}$,
  and that $h(R) \to 0$ as $R \uparrow \infty$ implies
  $L_m(\Theta_m) \to 0$.
\end{proof}

We now demonstrate the asymptotic
alignment $T_m/\|T_m\| - T^\star / \|T^\star\| \to 0$.
Define the mis-aligned region
\begin{equation*}
  \mc{R}(\epsilon) := \left\{T \mid
  \norm{ T/\|T\| - T^\star / \|T^\star\|}
  \geq \epsilon \right\}.
\end{equation*}
Let $\Theta = U\opt T$ for some $T \in \mc{R}(\epsilon)$, and
define the set
\begin{equation*}
  \mc{X}(T)
  \defeq
  \left\{x \in \mc{X} \mid \argmax_y \<\theta_y, x\> \neq \argmax_y
  \<\theta\opt_y, x\> \right\}.
\end{equation*}
Then we have the lower bound
\begin{align*}
  L_m(\Theta)
  & = \Ep [\sloss(\Theta^\top X, Y_m^+)]
  \ge \E[\sloss(\Theta^\top X, Y_m^+) \indic{X \in \mc{X}(T)}] \\
  & \ge \E\left[e_{y\opt(X)}^\top \rho_m({\Theta\opt}^\top X)
    \log\bigg(1 + \sum_{j \neq y\opt(X)}
    \exp(\<\theta_j - \theta_{y\opt(X)}, X\>)\bigg)
    \indic{X \in \mc{X}(T)} \right] \\
  & \ge \log 2 \cdot \E\left[e_{y\opt(X)}^\top \rho_m({\Theta\opt}^\top X)
    \indic{X \in \mc{X}(T)}\right],
\end{align*}
where we use that on the set $x \in \mc{X}(T)$,
at least one column $\theta_j$ satisfies
$\<\theta_j - \theta_{y\opt(x)}, x\> \ge 0$.
By dominated convergence, as $m \to \infty$,
\begin{equation*}
  \liminf_m L_m(\Theta)
  \ge \log 2 \cdot P(X \in \mc{X}(T)).
\end{equation*}
Because $\mc{X}(T)$ is a union of subspaces,
$T \mapsto P(X \in \mc{X}(T))$ is continuous and homogeneous in $\norm{T}$,
so that $\inf_{T \in \mc{R}(\epsilon)} P(X \in \mc{X}(T)) > 0$.

We have thus shown that
$\liminf_m \inf_{\Theta \in U\opt \mc{R}(\epsilon)}
L_m(\Theta) > 0$.
However, Lemma~\ref{lemma:norms-biggeroo}
shows that $\norm{\Theta_m} \to \infty$ and
$L_m(\Theta_m) \to 0$, so we must have
$\Theta_m \not \in U\opt \mc{R}(\epsilon)$ for large $m$,
and so $T_m / \norm{T_m} \to T\opt / \norm{T\opt}$.